\documentclass{article}

\usepackage{arxiv}

\usepackage[utf8]{inputenc} % allow utf-8 input
\usepackage{graphicx}
\usepackage[T1]{fontenc}    % use 8-bit T1 fonts
\usepackage{hyperref}       % hyperlinks
\usepackage{url}            % simple URL typesetting
\usepackage{booktabs}       % professional-quality tables
\usepackage{amsfonts}       % blackboard math symbols
\usepackage{nicefrac}       % compact symbols for 1/2, etc.
\usepackage{microtype}      % microtypography
\usepackage{lipsum}		% Can be removed after putting your text content
\usepackage{natbib}
\usepackage{doi}
\usepackage{amsmath}
\usepackage[ruled,vlined]{algorithm2e}
 \usepackage{doi}
\usepackage{tabularray}
\usepackage{makecell}
\usepackage{subcaption}
\usepackage{todonotes}
\usepackage{doi}
\usepackage{tabularray}
\usepackage{makecell}
\usepackage{float}
\usepackage{colortbl}
\usepackage{amsthm}
\usepackage{microtype} 
\usepackage{threeparttable}

\title{ On subbagging Boosted  Probit Model Trees}

\author{ \href{https://math.cas.lehigh.edu/tian-qin}{\hspace{1mm}Tian Qin} \\
	Department of Mathematics\\
	Lehigh University\\
	Bethlehem, PA 18015 \\
	\texttt{tiq218@lehigh.edu} \\
	\And
	\href{https://math.cas.lehigh.edu/wei-min-huang}{\hspace{1mm}Wei-Min Huang} \\
	Department of Mathematics\\
	Lehigh University\\
	Bethlehem, PA 18015 \\
	\texttt{wh02@lehigh.edu} \\
}

\renewcommand{\shorttitle}{\textit{arXiv} Template}

\begin{document}

\maketitle

\begin{abstract}%   <- trailing '%' for backward compatibility of .sty file
 With the insight of variance-bias decomposition, we design a new hybrid bagging-boosting algorithm named SBPMT for classification problems. For the boosting part of SBPMT, we propose a new tree model called Probit Model Tree (PMT) as base classifiers in AdaBoost procedure. For the bagging part, instead of subsampling from the dataset at each step of boosting, we perform boosted PMTs on each subagged dataset and combine them into a powerful "committee", which can be viewed an incomplete U-statistic. Our theoretical analysis shows that (1) SBPMT is consistent under certain assumptions, (2) Increase the subagging times can reduce the generalization error of SBPMT to some extent and (3) Large number of ProbitBoost iterations in PMT can benefit the performance of SBPMT with fewer steps in the AdaBoost part. Those three properties are verified by a famous simulation designed by \cite{Simudata}. The last two points also provide a useful guidance in model tuning. A comparison of performance with other state-of-the-art classification methods illustrates that the proposed SBPMT algorithm  has competitive prediction power in general and performs significantly better in some cases.
\end{abstract}

\keywords{Subagging \and ProbitBoost \and Weak consistency \and Incomplete U-statistic \and Generalization error bound}

\section{Introduction}

%We mainly focus on the classification problem in this paper. Ensembling results from many weak decision trees (DTs) improves the prediction accuracy by correcting for decision trees' habit of overfitting  the training set. Generally speaking, the original classification DTs  recursively partitions the feature space (spanned by covariates) into disjoint subspaces until each of them are nearly or completely homogeneous with respect to a particular class.
\label{Introduction}

Bagging has been a critical ensemble learning model for data classification or regression in past few decades. As an ensambling method, bootstrap bagging (\cite{Breiman2004BaggingP}) ensembles the results of base tree learners so that the variance of the model is decreased without increasing the bias. As a variant of bagging, Random Forests (RFs) further utilize a random subset of features at each possible split when constructing a single tree learner. \cite{RF}  analyzed how bagging and random subset of features can contribute to the accuracy gain under different conditions by decorrelating bootstrapped trees. \cite{Breiman2004BaggingP} suggested that bagging unstable (i.e high variance) classifiers usually generates a better result.

Instead of creating a strong learner by combining  weak learners fitted simultaneously, we can actually merge  weak classifiers into a strong classifier by sequentially revising them and adding them up with optimal weights, which gives rise to the idea of boosting. \cite{FREUND1997119} developed AdaBoost.M1 algorithm which outputs a weighted sum of a sequence of weak learners. The weights are tweaked in favor of those instances misclassified by previous learners. Nowadays, we have many other popular ensemble methods based on the idea of gradient boosting, which generalizes the boosting in functional spaces. For example, \cite{xgboostchen} followed the idea of gradient boosting and employed a different design of functional optimization and weak classifier learning algorithm. As a result, they proposed  one of the most successful ensemble learning algorithms called XGBoost. 

%LightGBM \cite{lightgbm} and Catboost \citeA{Ostroumova2017CatBoostUB}  can deal with categorical features by themselves and they are generally faster than XGboost . 

%In statistical view, boosting is a form of regression whose responses are some weak learner with respect to some covariates $\mathbf{x}$. For instance, the error function of AdaBoost is the exponential risk function:
%\begin{equation}
 %   E(f)=e^{-y(x)f(x)}
%\end{equation}

Empirically, if we iterate the boosting for a large number of steps, the training error of boosted classifier will be close to zero, which implies that boosting can somehow reduce the bias of original model. On the other hand, bagging can alleviate overfitting by averaging the outputs of base classifiers to reduce the variance of ensemble model. Observing the variance-bias decomposition in $L_2$ loss function, it's natural to combine bagging and boosting together to achieve smaller prediction error.

%If the performance of a learner is measured by $L_{2}$ loss, bagging somehow reduces the loss by decreasing variance without increasing the bias. While many literature illustrate that  boosting can achieve low bias without making variance too large. In some cases, it can further reduce the testing error when training error already reached  zero.

%More specifically, if we ensemble the random boosted classifiers, it is anticipated that the generated model can be more robust and achieve better performance.

\cite{adaptivebagging} proposed a hybrid bagging-boosting procedure for least-squared regression of additive expansion. \cite{sgb} incorporated randomness in gradient boosting where a subsample of training set is drawn without replacement at each iteration of gradient boosting. \cite{brf} introduced a boosting algorithm into Random Forest learning method to achieve high accuracy with the smaller size of forest. \cite{onestepbrf} proposed One-Step Boosted Forest for regression problem. They also gave a variance estimate of Boosted Forest which makes the statistical inference of ensemble model possible.

However, for "boosting-type" method, we typically need hundreds of iteration times to turn a weak learner into a stronger one. Similarly, for "bagging-type", such as Random Forests, hundreds or thousands of weak learners are required to achieve a decent prediction accuracy.  It's hard for practitioners to follow the performance of each weak learner when they analyze the model and try to interpret the results.  \cite{brf}  illustrated that combining boosting and random forest is  helpful in reducing the number of base learners required for boosted random forest. Following these clues, we introduce a model based bagging-boosting procedure whose simple but efficient structure could achieve high prediction accuracy with small number of base learners. 

%make the interpretation of ensemble methods possible

As to the interpretability of the classifier, if the base structure of a classifier is decision tree, which is commonly used in many ensemble methods, employing generalized linear models (GLM) at terminal nodes could be a good choice because of its flexibility. The fitting process of a GLM is to maximize its corresponding sample log-likelihood function by an iteratively re-weighted least-squares(IRWLS) approach:
\begin{equation}
    \mathcal{\beta}^{(t+1)}=\mathcal{\beta}^{(t)}+\mathcal{J}^{-1}(\mathcal{\beta}^{(t)})u(\mathcal{\beta}^{(t)})
\end{equation}
where $\mathcal{\beta}^{(t+1)}$ is updated vector of parameters at step $t$, $\mathcal{J}^{-1}(\mathcal{\beta}^{(t)})$ is the observed information matrix and $u(\mathcal{\beta}^{(t)})$ is the score function. However, if we fit a GLM by all features at each terminal node of tree, it's highly likely that the information matrix  will have ill-posed condition as the number of observations at a terminal node is typically small. As a result, matrix inversion during fitting procedure is unavailable. To avoid those drawbacks, we can assume an additive form of GLMs. \cite{logitboosting} proposed  LogitBoost algorithm by employing the additive model into logistic regression problem. If the weighted least square regression is taken as the weak learner, LogitBoost will add one feature at a time in boosting iteration, which avoids matrix inversion. Logitboost also keeps the interpretability since it models the half log-odds by a linear combination of selected features.  \cite{RN50} first combined  LogitBoost and decision tree together and designed a tree model called Logistic Tree Model(LMT). Similarly, \cite{ZHENG20124428} proposed ProbitBoost which fits the GLM with Probit link. In practice, GLM with probit link and logistic link have similar performance under most contexts. While in the theoretical point of view, \cite{ZHENG20124428} pointed out that the re-weighting function of LogitBoost will give a wrongly classified example a small weight making it difficult to correct the error in the later iterations. In ProbitBoost, the re-weighting step will assign larger weight for mis-classified examples, which implies ProbitBoost has better self-correction ability. In addition, \cite{logitboosting} required cross-validation to determine optimal LogitBoost iteration steps at each inner and terminal node when constructing a LMT. Due to the nature of cross-validation, we will sacrifice part of training data in order to determine a hyperparameter. When data size is relatively small, the benefit of cross-validation may be overturned by the loss of effective data. Based on those observations, we'd like to introduce our novel base tree learner named  Probit Model Tree(PMT) which combines ProbitBoost and decision tree together. By employing PMT as base learner in our bagging-boosting structure, we can further show that there is no need to control the iteration times in ProbitBoost part. That implies  hyperparameter tuning can be more straightforward and efficient since  cross-validation will be avoided in fitting process of our proposed model. 

%However, as we will see in next section, during the reweight step of logitboost, there may be numerical issues in practical implementation. \citeA{ZHENG20124428} proposed ProbitBoost which doesn't have the numerical issues in logitboost but has the similar performance with LogitBoost. Based on those results, we will introduce a novel base tree learner named  Probit Model Tree(PMT). 

The remaining sections of this paper are organized as following schema. In section 2, we introduce some related works for some hybrid bagging-boosting methods.  In section 3, we illustrate the idea of Probit model tree(PMT). After that our new bagging-boosting method named SBPMT will be proposed in section 4, along with some theoretical analysis in section 5. We also compare the performance of SBPMT in  real datasets in section 6. An experiment with simulation data will be demonstrated as well.  Section 7 discusses few characteristics of SBPMT and proposes future directions. The R codes for the implementation of SBPMT, selected real datasets experiments and the simulation are available on Github.\footnote{https://github.com/BBojack/SBPMT}

\section{Related works for bagging-boosting methods}
\label{network}

%\subsection{Logistic Model Tree}

 %In the Logistic model tree (LMT) model, \cite{LMT} follows the same way as CART to do the partition. But the novelty of their method is they constructs the logistic model at lower level in the tree by incrementally refining the logistic model already fit at higher levels. More specifically, assume that we have fit a logistic regression model by LogitBoost at a parent node and want to build the logistic regression function at child nodes.  They then use the fitted parameters and weights in parent node as the initial value of the logistic boosting model for child nodes. This step is viewed as encoding "global" influences of the feature they choose to the split. At last, they used the fitted logistic model in terminal nodes to predict the class of observations. Such hybrid type of trees motivates us to employ parametric models at terminal nodes of trees. 

 Many popular  boosting algorithms such as Stochastic Gradient Boosting (\cite{sgb}), CatBoost (\cite{Ostroumova2017CatBoostUB}), LightGBM (\cite{lightgbm}) and XGBoost (\cite{xgboostchen}) can actually incorporate bagging during the iteration in their implementation. That is, drawing a subsample of training set based on the weights calculated at each iteration of boosting  to reduce the computation cost and avoid noise in data.  In that way, the bagging procedure is dependent upon boosting procedure and performed sequentially. As a result, not whole training set will be utilized during the iteration. It can be dangerous to ignore or pay less attention to part of the training data when we have relatively small size of dataset. And in practice, dosing randomness at each step of boosting forces practitioner to set large number of iteration times to make sure the generated classifier is stable and strong enough. As a result, the final model involves extremely large number of weak learners and the model interpretation becomes nearly impossible.
 
 \cite{brf} proposed Boosted Random Forest which introduces boosting into random forest. More specifically, they sequentially added a decision tree in the forest by AdaBoost. As a result, less number of trees are required to maintain discrimination performance. They constructed a decision tree by random sampling of features and splitting value to achieve the high generality of the boosted random forest. But the price paid is the loss of interpretation. It becomes even harder to measure the variance importance due to these two layers of randomness.

 Recently,  \cite{onestepbrf} proposed One-Step Boosted Forest for regression problem. They also gave a variance estimate of Boosted Forest which makes the statistical inference of ensemble models possible.  Their empirical results imply that taking independent subsamples to construct ensemble model will generate better prediction than repeating the same subsamples. Although their work is mainly about regression problem, that observation motivates us to build a classifier by taking different possible subsamples, which is exactly what subagging means.

\section{Probit Model Tree}
\label{mv-heuristic}

In this section, we will give a brief introduction of the idea of ProbitBoost (\cite{ZHENG20124428}) which is the core part of our algorithm. Next, the whole procedure of PMT algorithm will be presented, which combines the ordinary  classification decision tree with ProbitBoost model at each node of a decision tree.

\subsection{ProbitBoost for Binary class}

Suppose we have a training set with $n$ observations $\{(\mathbf{x}_{i},Y_{i})_{i=1,...,n}\}$ where each item is associated with class $Y_{i}\in \{-1,1\}$. $\mathbf{x}_{i}$ is the feature vector of $i$-th observation and we assume that each observation has dimension $p$.

In ordinary probit regression model (binary classification problem), we model the posterior class probability $P(Y=1|X=\mathbf{x}_{i})$ of $i$-th observation by linking additive linear combination of features with the normal cdf (cumulative distribution function) as follows:
\begin{equation}
\label{eq1}
    P(Y=1|X=x_{i})=\Phi(\sum_{j=1}^{p}\beta_{j}x_{ij})=\Phi(\beta^{T}\mathbf{x}_{i}) , \quad j=1,...,p
\end{equation}
where $p$ is the dimension of feature space , $\Phi(\cdot)$  is the cdf of normal distribution and $x_{ij}$ represents the $j$-th feature of $i$-th observation. The main assumption is the linear form of  features at the right hand side of \eqref{eq1}. More generally, we can take any functional forms of the input feature space to capture  complicated patterns in a given dataset.

As a result, the generalized additive form can be written as
\begin{equation}
    F_{M}(\mathbf{x})=\sum_{m=1}^{M}f_{m}(\mathbf{x})
\end{equation}
for each m,  $f_{m}(\mathbf{x})$ is a $R^{p} \to R$ function and the capital $M$ is the total number components of the general additive model, which makes it possible to view the boosting procedure as a forward stagewise  fitting algorithm of additive models. In many contexts, functions $f_{m}(\mathbf{x})$ are taken to be some simple basis functions (mostly linear) so that the fitting process is tractable and easy to interpret. 

We now propose population probit risk function as following:
\[
P(f(\mathbf{x}))=-E[log(\Phi(Yf(\mathbf{x})))]
\]
and the empirical version is:
\[
\quad P_{n}(f(\mathbf{x}))=-\frac{\sum_{i=1}^{n}log(\Phi(Y_{i}f(\mathbf{x_{i}}))}{n}
\],
where $Y_{i}$ takes the value in $\{-1,1 \}$.

Like other risk functions such as exponential risk and Logit risk, Probit risk function is convex and second-order differentiable.

Following the structure of ProbitBoost proposed by  \cite{ZHENG20124428}, we give a brief derivation of ProbitBoost algorithm under Newton-Raphson framework and Probit risk function. Reader can find similar details in their work.

Suppose we are minimizing the populational probit risk function $P(f(\mathbf{x}))=-E[log(\Phi(Yf(\mathbf{x})))]$, then the fitting process proceeds by the Newton-Raphson iterative equation as follows:
\begin{equation}
    f^{(m+1)}(\mathbf{x})=f^{(m)}(\mathbf{x})-H^{-1}(f^{(m)}(\mathbf{x}))D(f^{(m)}(\mathbf{x}))
\end{equation}
where $f^{(m)}(\mathbf{x})$ is the fitted function at $m$-th step. This setting coincides with the idea of forward stagewise additive model which makes the implementation of algorithm convenient. Denote $D(), H()$ be the first and second order derivative of expected log-likelihood function , respectively.

Under appropriate conditions, we can find out the gradient (first order derivative) of $P(f(\mathbf{x}))$ w.r.t $f(\mathbf{x}))$ is :
\begin{equation}
\begin{split}
    D(f)=-E\bigg[  \frac{Y\phi(f)}{\Phi(Yf)}  \bigg|\mathbf{x}\bigg]
\end{split}
\end{equation}
Similarly, the Hessian( second order derivative) of $P(f(\mathbf{x}))$ is calculated as
\begin{equation}
\begin{split}
    H(f)=E\bigg[  \frac{\phi(f)G_{Y}(f)}{\Phi^{2}(Yf)}  \bigg|\mathbf{x}\bigg]=E[h(f) |\mathbf{x}]
\end{split}
\end{equation}
where $G_{Y}(f)=Yf\Phi[Yf]-\phi(Yf)$.

Then the Newton-Rahpson iteration process can be written as
\begin{equation}
\label{eq7}
    \begin{split}
         f^{(m+1)}(\mathbf{x})&=f^{(m)}(\mathbf{x})-H^{-1}(f^{(m)}(\mathbf{x}))D(f^{(m)}(\mathbf{x}))\\
         &=f^{(m)}(\mathbf{x})-\frac{1}{E[h(f^{(m)})|\mathbf{x}]}\bigg(-E\bigg[ \frac{Y\phi(f^{m})}{\Phi(Yf^{m})}  \bigg|\mathbf{x} \bigg]\bigg) \\
         &=f^{(m)}(\mathbf{x})+E_{h}\bigg[ \frac{Y\Phi(Yf^{(m)})}{G_{Y}(f^{(m)})}   \bigg|\mathbf{x}\bigg]
    \end{split}
\end{equation}
where $E_{W}(\cdot |\mathbf{x})$ is the notation for weighted conditional expectation defined in \cite{Logitboost}, which means 
\begin{equation}
\label{eq8}
   E_{w}[g(\mathbf{x},y)|\mathbf{x}]=\frac{E[w(\mathbf{x},y)g(\mathbf{x},y)|\mathbf{x}]}{E[w(\mathbf{x},y)|\mathbf{x}]} \quad  \text{with} \quad w(\mathbf{x},y)>0 \quad \forall x \forall y. 
\end{equation}
where $g(\mathbf{x},y)$ is an appropriate function with finite weighted conditional expectation as defined in \eqref{eq8}.
In the last equation of \eqref{eq7}, the weight is $h(f)$. We can show that $h(f)$ is positive all the time so it makes sense to use $h(f)$ as the weight function.

The the sample version of ProbitBoost is illustrated in Algorithm \ref{algo-1} below:

\begin{algorithm}[H]
\SetAlgoLined

 Initialization: Total iteration number $M$; training data $\{(\mathbf{x_{i}},Y_{i}),i=1,...,n\}; Y_{i}\in \{-1,1\} $;
 
   Start with weights $w_{i}=1/n,i=1,2,...,n$, $f^{0}(\mathbf{x})=0$ and class probability $p(\mathbf{x_{i}})=\frac{1}{2}$\;
  
 \For{$m=1,2,...,M$}{
 
 (a) Compute the working response $\mathbf{z}=(z_{1},...,z_{n})$  as 
 
 \[
 z_{i}=\frac{(Y_{i}\Phi(f^{m-1}(x_{i})))}{G_{Y_{i}}(f^{m-1}(x_{i}))}
 \]
 
 for $i=1,2,...,n$\;
 
  (b) Compute the weights $w$  as 
 
 \[
 w_{i}=\frac{\phi(f^{m-1}(x_{i}))G_{Y_{i}}(f^{m-1}(x_{i}))}{\Phi^{2}[Y_{i}f^{m-1}(x_{i})]}
 \]
 
 for $i=1,2,...,n$\;
 
 (c) Fit the current working response $\mathbf{z}$ by a weak learner $g^{m}(\mathbf{x})$ with weights $w_{i}$, $i=1,2,...,n$.\;
 
 (d) Update $ f^{m}(\mathbf{x})=f^{m-1}(\mathbf{x})+g^{m}(\mathbf{x})$

 }
 \KwResult{Output the classifier $sign[f^{M}(\mathbf{x})]$ }
 \caption{ProbitBoost (Binary classification)}
 \label{algo-1}
\end{algorithm}

In next section we will give a specific example of weak learner $g^{m}(\mathbf{x})$.

\subsection{Weighted least square regressor as weak learner}

In step (c) of Algorithm \ref{algo-1}, we need to find a optimal weak learner so that the probit risk function decreases the most. This idea is a kind of coordinate descent method in optimization. Weighted least square regression or weighted decision tree are two main options for the weak learner. \cite{reg_stump} applied regression stump in boosting. Similarly, \cite{ZHENG20124428}  used regression stump in their ProbitBoost model and found that regression stump performs better in many cancer and gene expression classification datasets. However, in tabular datasets we found that regression stump has its own pitfalls in tree models. The criterion of choosing best feature also deserves discussing. \cite{betaboost} selected the attribute by the largest increment of sample log-likelihood function. 

Of course the weak learner can be non-linear. For instance, we can select feature by kernel regression. However, one main drawback of using non-linear weak learner as weak learner is, we may lose the interpretability of the boosted classifier by introducing complicated structure. Meanwhile, especially for kernel regression, the relationship between transform map and kernel is not bijective. Despite the flexibility of choosing a particular kernel, it's hard to find the true transform map which makes interpretation more difficult.

We decide to apply weighted least square (WLS) regression as weak learner in our ProbitBoost setting which makes interpretation possible and the time efficiency is guaranteed by the formula of simple WLS without involving matrix inversion. Note that the final fitted function $f^{M}(\mathbf{x})$ will have the form of linear combination of features. We now summarize details in step(c) at $m$-th iteration:

%We select the optimal weak learner (optimal direction/feature) which has the smallest  weigthted square error with  WLS at $m$ th iteration:

    \begin{itemize}
        \item c(1) Perform weighted least square regression for each feature  with weights $w_{i}$ to obtain:
        \[
(\hat{a}_{j},\hat{b}_{j})=\operatorname*{arg\,min}_{a,b} \sum_{i=1}^{n}w_{i}(z_{i}-ax_{j,i}-b)^{2} \quad \text{for $j\in\{1,...,p\}$}
\]
        \item c(2)  Obtain the index $s(m)$ of optimal weak learner (optimal direction/feature)  which has the smallest  weighted square error:
        \[
s(m)=\operatorname*{arg\,min}_{j\in \{1,...,p\}} \sum_{i=1}^{n}w_{i}(z_{i}-\hat{a}_{j}x_{j,i}-\hat{b}_{j})^{2}
\]
        \item c(3)  Set $g^{m}(\mathbf{x})=\hat{a}_{s(m)}\mathbf{x}+\hat{b}_{s(m)}$
    \end{itemize}

where vector $\mathbf{w}=\frac{\phi(Yf^{m-1})[Yf^{m-1}\Phi(Yf^{m-1})+\phi(Yf^{m-1})]}{\Phi^{2}(Yf^{m-1})}, \mathbf{z}=\frac{\phi(Yf^{m-1})}{Yf^{m-1}\Phi(Yf^{m-1})+\phi(Yf^{m-1})}$ and $\Tilde{\mathbf{x}}=(1,\mathbf{x})$. 

Eventually, the fitted additive probit regression function after $M$ iterations can  now be represented as $f^{M}(\mathbf{x})=\Tilde{\mathbf{x}}^{T}(\sum_{i=1}^{M}\beta_{i})=\Tilde{\mathbf{x}}^{T}\mathbf{\beta}$ where $\beta_{i}\in R^{p+1}$ and $\beta_{i,1}=b_{s(i)},\beta_{i,s(i)+1}=a_{s(i)},\beta_{i,j}=0$ for $j\neq 1,s(i)+1$.

\subsection{Growing the  probit model tree}

For simplicity and efficiency, we follow the spirit of CART algorithm to build a decision tree first and and fit a probit model with ProbitBoost algorithm described in section 3.1 at each terminal node for a tree. We don't take specific regularizations to control the complexity of models as in section 5 we will see there is actually no need to control the iteration times for ProbitBoost. More specifically, regardless of time costing, large number of iterations for ProbitBoost not only benefits the empirical probit risk but also the prediction accuracy. The stopping rules includes min\_leaf\_size, max\_depth and max\_iterations of ProbitBoost. In reality, user can set those parameters by themselves based on the specific condition of experiments. 

\subsection{The model}

As specified in section 3.1-3.3, a probit model tree consists of two parts. The global structure is a decision tree with CART type. And the local structure, which means a model at each terminal node of such a s tree, is a probit regression model fitted by ProbitBoost. 

Suppose $\{S_{a}\}_{a\in A}$ is a set of partitions determined by the CART decision tree structure, where $A$ is the total number of partitions. Each partition $S_a$ has an associated additive ProbitBoost model   $G_{a}$  generated by Algorithm 1. Then  the whole probit model tree is represented by 
   \begin{equation}
          G(\mathbf{x})=\sum_{a\in A}sign[G_{a}(\mathbf{x})]\cdot I(\mathbf{x}\in S_{a})
   \end{equation}
   
where $I(\mathbf{x}\in S_{a})=1$ if $\mathbf{x}\in S_{a}$ and 0 otherwise, $\mathbf{x}$ is a data point.

As we take the weighted least square regressor as the weak learner, we can further expand each $G_{a}(\mathbf{x}) $ as following:

\[
G_{a}(\mathbf{x})=\mathbf{x}^{T}\beta^{a}
\]
where $\beta^{a}$ is the vector of coefficients fitted in accordance with the procedure described in section 3.2.

Now the most detailed form of a probit model tree is written as
\begin{equation}
\begin{split}
        G(\mathbf{x})&=\sum_{a \in A}sign[\mathbf{x}^{T}\beta^{a}]\cdot I(\mathbf{x}\in S_{a})\\
        &=sign[\sum_{a\in A}\mathbf{x}^{T}\beta^{a}\cdot I(\mathbf{x}\in S_{a})]
\end{split}
\end{equation}
The second equality holds because of the fact that  partitions $\{ S_{a}  \}_{a \in A}$ are disjoint with each other.

Algorithm \ref{algo-pmt} summarizes the whole procedure for fitting a PMT.

\begin{algorithm}[H]
\SetAlgoLined

     Initialization: Number of ProbitBoost iteration $B$ ; Depth of PMT $l$; training data $\mathcal{D}_{X}=\{(\mathbf{x}_{i},Y_{i}),i=1,...,n\}$ with $\mathbf{x}_{i} \in \mathbb{R}^{p}$ and $Y_{i}\in \{-1,1\} $.

($a$) Fit a CART tree $T$ with maximal depth $l$ and obtain the corresponding partition $\{S_{a}\}_{a\in A}$

($b$) \For{$a \in A$}{
     Perform ProbitBoost with iteration number $B$ for data points lying in the partition $S_{a}$ and obtain 
 the classifier $G_{a}(\mathbf{x})=\mathbf{x}^{T}\beta^{a}$.

}  
 \KwResult{Output the final classifier $G(\mathbf{x})=sign[\sum_{a\in A}\mathbf{x}^{T}\beta^{a}\cdot I(\mathbf{x}\in S_{a})]$ }
 \caption{Probit Model Trees (PMT) }
 \label{algo-pmt}
\end{algorithm}

\section{Subagging Boosted  Probit Model Trees (SBPMT)}

Having introducing the idea of Probit Model Tree, in this section, we will give a new hybrid bagging-boosting algorithm based on PMTs. We consider the binary case first. As mentioned before, intuitively, bagging-boosting algorithm could improve the mean squared error simultaneously in two directions. Instead of trading off variance and bias by regularization, bagging-boosting type algorithm could reduce both of them at the same time. For classification problems, such observation encourages us to implement bagging-boosting type classifiers to reduce the misclassification error.

We first give a brief introduction of subagging. As a variant of the bagging procedure, subagging, which is first proposed by \cite{subagging}, is a sobriquet for "subsample aggregating" where the bootstrap resampling is replaced with subsampling.  A learner or predictor $F_{n,m}(\mathbf{x})$ with subsampling size $m$ is aggregated as follows:
\begin{equation}
    F_{n,m}(\mathbf{x})=\frac{1}{\binom{n}{m}}\sum_{(i_{1},...,i_{m})\in l}h_{m}(L_{i_{1}},...,L_{i_{m}})(\mathbf{x})
\end{equation}
where $l$ is the set of all $m$-tuples whose elements in $\{1,2,...,n\}$ are all distinct and $h_{m}(L_{i_{1}},...,L_{i_{m}})$ is a predictor based on a subsampling with indices $(i_{1},...,i_{m})$.
But in practice, it's time consuming to evaluate all $\binom{n}{m}$ subagged predictors. One way to avoid heavy computation burden is to approximate (11) by  random sampling . That is, we randomly sample without replacement from original dataset for $M$ times and averaging over the predictors based on random samplings. The subagged predictor becomes:
\begin{equation}
    F_{n,m,M}(\mathbf{x})=\frac{1}{M}\sum_{i}^{M}h_{m}(L_{i_{1}},...,L_{i_{m}})(\mathbf{x})=\frac{1}{M}\sum_{i}^{M}h_{m}(L_{D_{i}})(\mathbf{x})
\end{equation}
where $(L_{i_{1}},...,L_{i_{m}})=L_{D_{i}}$. We will take this type of subagging in our algorithm.  Note that $\mathcal{D}=(D_1,...,D_{M})\in \{ D\subset [n]:|D|=m\}^{M}$ is a sequence of $M$ subsets of $\{1,2,...,n\}$ with cardinality $m$ and this sequence $\mathcal{D}$ is called the design \cite{DesignBasedIncomUstat}.  It's easy to see $F_{n,m,M}(\mathbf{x})$ is an incomplete $U$-statistic with order $m$ for most learning algorithms. For simplicity, we take subsampling size $m=\lfloor \alpha n \rfloor$ with $0<\alpha <1$.
\begin{figure}[h]
\includegraphics[width=0.9\textwidth]{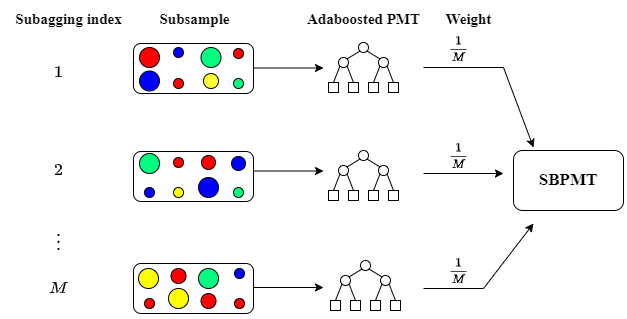}
\caption{Working flow of the SBPMT algorithm}
\label{fig-sb}
\end{figure}
The main idea of SBPMT is to perform AdaBoost at each subagged sampling. PMTs will be used as base learner in Adaboost. Note that the individual weights generated in AdaBoost are employed in both CART algorithm and ProbitBoost for fitting a PMT. Figure \ref{fig-sb} illustrates the working flow of SBPMT.

Algorithm \ref{algo-2} summarizes the details of the SBPMT algorithm. At step ($b$) we relabel the indices of subsampled dataset from 1 to $\lfloor \alpha n \rfloor$ for the sake of simplicity. 

\begin{algorithm}[H]
\SetAlgoLined

     Initialization: Number of subagging $M$; Number of AdaBoost iteration $T$ ;Number of ProbitBoost iteration $B$ ;subagging ratio $\alpha$; Depth of PMT $l$; training data $\mathcal{D}_{X}=\{(\mathbf{x}_{i},Y_{i}),i=1,...,n\}; Y_{i}\in \{-1,1\} $;
  
 \For{$k=1,2,...,M$}{
 
 ($a$) Randomly sample from the original training data $\mathcal{D}_{X}$ without replacement. The size $m$ of random sampled dataset $\mathcal{D}_{X}^{k}$  is set to $\lfloor \alpha n \rfloor$.   Initialize weights $w_{1}(i)=1/m,i=1,2,...,m$
 
  ($b$) On $\mathcal{D}_{X}^{k}$, \For{$t=1,2,...,T$}{
 
  ($b_{1}$) Fit a PMT classifier $G_{t}(\mathbf{x})$ to the subagged dataset $\mathcal{D}_{X}^{k}$ using weights $w_{t}(i)$.
     
  ($b_{2}$) Compute $\alpha_{t}=\frac{1}{2}ln\frac{1-err_{t}}{err_{t}}$ where $err_{t}=\sum_{i=1}^{m}w_{t}(i)I(Y_{i}\neq G_{t}(\mathbf{x}_{i}))$
  
  ($b_{3}$) Update weights: $w_{t+1}(i)=\frac{w_{t}(i)\cdot e^{\alpha_{t}\cdot I(Y_{i}\neq G_{t}(\mathbf{x}_{i})}}{Z_{t}}$ where $Z_{t}$ is a normalization factor so that $w_{t+1}(i)$ becomes a distribution, $i=1,2,...,m$.
 }

  ($c$)Output $f_{Z_k,\lfloor \alpha n \rfloor,T}(\mathbf{x})=sign[\sum_{t=1}^{T}\alpha_{t}G_{t}(\mathbf{x})]$, where $f_{Z_{k},\lfloor \alpha n\rfloor,T}$ be fitted the AdaBoosted PMT with $T$ iterations fitted by $k$ randome sampled dataset $\mathcal{D}_{X}^{k}$.

 }
 \KwResult{Output the final classifier $F(\mathbf{x})=sign[\frac{1}{M}\sum_{k=1}^{M}f_{Z_k,\lfloor \alpha n \rfloor,T}(\mathbf{x})]$ }
 \caption{Subagging Boosted Probit Model Trees (SBPMT)}
 \label{algo-2}
\end{algorithm}

\subsection{Multi-class version}

SBPMT can be generalized to multi-class problems by using multi-class versions of AdaBoost and Probit Model Tree. \cite{ZHENG20124428} proposed Multi-class gradient ProbitBoost by maximizing the multiclass log-likelihood function for probit regression.  They also sketched one-versus-all approach to fit a multi-class ProbitBoost classifier. The idea is to reduce a multi-class problem into $J$ binary classification problems where $J$ is the number of classes. Due to the simplicity of one-versus-all strategy, we will fit a PMT by one-versus-all strategy instead of fitting multi-class ProbitBoost at each terminal node of a CART tree. 

Having a multi-class version of PMT, we incorporate it in a multi-class version  AdaBoost  named SAMME which is proposed by \cite{Hastie2009MulticlassA}. SAMME extends the AdaBoost to the multi-class case without reducing it to multiple binary problems. It is equivalent to forward stagewise additive modeling with a multi-class exponential loss function. We choose this multi-class version of AdaBoost because of its clean implementation. In practice, people can try many other variants of multi-class AdaBoost.

%\citeA{Hastie2009MulticlassA} proved that the multi-class exponential loss function is Fisher consistent but we won't discuss too many details here.

We now sketch the multi-class version of SBPMT as following:

\begin{algorithm}[H]
\SetAlgoLined

     Initialization: Number of ProbitBoost iteration $B$ ; Depth of PMT $l$; training data $\mathcal{D}_{X}=\{(\mathbf{x}_{i},Y_{i}),i=1,...,n\}$ with $\mathbf{x}_{i} \in \mathbb{R}^{p}$ and $Y_{i}\in \{1,2,...,J\} $; Define $y_{ij}=I(Y_{i}=J)$ for $i=1,2,...,n$ and $j=1,2,...,J$.
  
 \For{$k=1,2,...,M$}{
 
 ($a$) Randomly sample from the original training data $\mathcal{D}_{X}$ without replacement. The size $m$ of random sampled dataset $\mathcal{D}_{X}^{k}$  is set to $\lfloor \alpha n \rfloor$. Initialize weights $w_{1}(i)=1/m,i=1,2,...,m$.
 
  ($b$) On $\mathcal{D}_{X}^{k}$, \For{$t=1,2,...,T$}{
 
  ($b_{1}$) Fit a multi-class PMT classifier $G_{t}(\mathbf{x})$ by Algorithm \ref{algo-pmt} using weights $w_{t}(i)$.

  ($b_{2}$) Compute $\alpha_{t}=\frac{1}{2}ln\frac{1-err_{t}}{err_{t}}+log(J-1)$ where $err_{t}=\sum_{i=1}^{m}w_{t}(i)I(Y_{i}\neq G_{t}(\mathbf{x}_{i}))$
  
  ($b_{3}$)Update weights: $w_{t+1}(i)=\frac{w_{t}(i)\cdot e^{\alpha_{t}\cdot I(Y_{i}\neq G_{t}(\mathbf{x}_{i})}}{Z_{t}}$ where $Z_{t}$ is a normalization factor so that $w_{t+1}(i)$ becomes a distribution, $i=1,2,...,m$
 }

  ($c$)Output $f_{Z_k,\lfloor \alpha n \rfloor,T}(\mathbf{x})=\underset{j \in \{1,2,...,J\}}{\mathrm{argmax}} \sum_{t=1}^{T}\alpha_{t}I(G_{t}(\mathbf{x})=j)$.

 }
  
 \KwResult{ Output the final classifier $F(\mathbf{x})=\underset{j \in \{1,2,...,J\}}{\mathrm{argmax}} (\sum_{k=1}^{M}I(f_{Z_k,\lfloor \alpha n \rfloor,T}(\mathbf{x})=J)) $ }
 \caption{Multi-class SBPMT}
 \label{algo-3}
\end{algorithm}

\section{Theoretical analysis of SBPMT}
\label{analysis}

In this section, we will show the Bayesian (weak) consistency w.r.t SBPMT under some reasonable assumptions. A simple but useful fact that the consistency of individual classifier is preserved by majority voting, which is applied in the ensembling step of SBPMT,  has been shown in \cite{GAO2022103788} and \cite{consisrf}. For completeness of our work, we  gave a more elementary proof. After that, we proved the consistency of AdaBoosted PMT using the results from \cite{bartlett07b}. 

%With all ingredients prepared, the consistency of SBPMT follows straightforward ({\color{blue} Do we need this sentence?}).

%({\color{blue} Consistency is a desirable property.}) In practice, the consistency doesn't help too much in reducing the generalization error of SBPMT. 

In order to achieve decent performance of SBPMT in real datasets, we derived an upper bound for the  generalization error  of SBPMT in section 5.3, which illustrates the impact of each hyperparameter on the generalization error of SBPMT. 

\subsection{Setup and notation}

Again, for simplicity, we consider a binary classification problem. Let $\mathcal{X}$ be the feature(measurable) space and $\mathcal{Y}=\{-1,1\}$ be the set of labels. Suppose we observe a training data $S_n=\{(X_i,Y_i)\}_{i=1}^{n}$ of i.i.d observations from an unknown distribution $\mathcal{P}$. We now construct a classifier $g_n:\mathcal{X}\to \mathcal{Y}$ based on this training set. Then the probability of error of $g_n$ is given by 

\[
L_{n}=L(g_n)=P(g_{n}(X)\neq Y|S_n)
\]
which is also called empirical risk. 

In theory \cite{Devroye1996APT}, the expected risk of a given classifier $g:\mathcal{X}\to \mathcal{Y}$ is defined as

\[
L(g)=E[I(g(X)\neq Y)]
\]

The optimal expected risk (or Bayes risk) $L^{*}$ is the  infimum over all possible classifiers g:

\[
L^{*}=\inf_{g}L(g)=E[min(\eta(X),1-\eta(X))]
\]

where $\eta(\cdot)$ is a conditional probability $\eta(x)=P(Y=1|X=x)$.

It's well known that the infimum above can be achieved by the Bayes classifier  $g^{*}(x)=g(2\eta(x)-1)$:

\[
g(x)=
    \begin{cases}
        1, & x>0\\
        -1, & x \leq 0
    \end{cases}
\]

Note that $g^{*}$ relies on the distribution $\mathcal{P}$ of $(X,Y)$ and expectation is taken over that unknown distribution $\mathcal{P}$. 
When the sample size increases, we can have a sequence of classifiers $\{g_{n},n\geq 1\}$, which is denoted as a rule.  Intuitively, we want a classifier have small empirical risk and be close to the Bayes risk if we have enough data. Statistically speaking, a rule is consistent under a certain distribution $\mathcal{P}$ if:

\[
\lim_{n\to \infty}E[L(g_{n})]=L^{*}
\]
\newtheorem*{remmark}{Remark}
\begin{remmark}
Since $L_{n}$ is a random variable which is bounded, convergence of the expected value is equivalent to the convergence in probability, which means that 
\[
\lim_{n\to \infty}P(|L(g_{n})-L^{*}|>\epsilon)=0
\]
for all $\epsilon>0$.
\end{remmark}

For the subagging classifier which takes a majority vote over  classifiers built upon $M$ random subsamples with subsampling ratio $\alpha$, we can define $f_{M,n,T}(\mathbf{x})$ as following:
\begin{equation}
\label{eq-voting}
    f_{M,n,T}(\mathbf{x})=
    \begin{cases}
        1, & \frac{1}{M}\sum_{i=1}^{M}f_{Z_{i},\lfloor \alpha n\rfloor,T}(\mathbf{x})> 0\\
        -1, & \text{otherwise}
    \end{cases}
\end{equation}
where $(Z_1,...,Z_M)$ are i.i.d random variables (let $\mathcal{Z}$ be their common distribution) representing the randomization introduced by subsampling. In our case, $f_{Z_{i},\lfloor \alpha n\rfloor,T}$ is the AdaBoosted PMT with $T$ iterations, which is constructed by $i$-th subagged sample. 

Next, we follow the settings in \cite{bartlett07b}.

Denote  $\mathcal{H}=\{h|h:\mathcal{X}\to \mathcal{Y}=\{-1,1\}\}$  as the hypothesis space and $d_{VC}(\mathcal{H})$  as VC (Vapnik-Chervonenkis) dimension.

The convex hull of $\mathcal{H}$ scaled by $\lambda \geq 0$ is

\[
\mathcal{F}_{\lambda}=\bigg\{ f\bigg| f=\sum_{i=1}^{n}\lambda_{i}h_{i},n\in \mathcal{N}\cup \{0\} ,\lambda_{i}\geq 0,\sum_{i=1}^{n}\lambda_{i}=\lambda ,h_{i}\in \mathcal{H}\bigg\}
\].

Define 
\[
R_{n}(f)=-\frac{1}{n}\sum_{i=1}^{n}e^{-Y_{i}f(X_{i})}\quad \text{and} \quad R(f)=E[e^{-Yf(X)}].
\]
Then the Adaboost procedure can be summarized as follows.

\begin{itemize}
    \item 1. Set $f_{0}\equiv 0$, the number of iterations t and weights $w_{1}(i)=1/n,i=1,2,...,n$.
    \item 2. For $k=0,...,t$, the updating equation is 
    \[f_{k+1}=f_{k}+\alpha_{k}h_{k}\]
    where $f_{k}$ satisfies $R_{n}(f_{k})= \displaystyle \inf_{h\in \mathcal{H},\alpha \in \mathbb{R}}R_{n}(f_{k-1}+\alpha h)$. 
    More specifically, 
    \begin{itemize}
        \item 2(a). Fit a classifier to the training data with weights $w_{i}$ to obtain optimal base learner $h_{k}$
        \item 2(b). Compute $\varepsilon_{k}=\sum_{i=1}^{n}w_{k}(i)I(y_{i}\neq f_{k}(x_{i})$.
        \item 2(c). Compute step size $\alpha_{k}=\frac{1}{2}ln\bigg(\frac{1-\varepsilon_{k}}{\varepsilon_{k}}\bigg)$.
        \item 2(d). Update weights: $w_{k+1}(i)=\frac{w_{k}(i)\cdot e^{\alpha_{k}\cdot I(y_{i}\neq f_{k}(x_{i})}}{Z_{k}}$ where $Z_{k}$ is a normalization factor so that $w_{k+1}(i)$ becomes a distribution.
    \end{itemize}
    \item 3. Return $g\circ f_{t}$ as a final classifier.
\end{itemize}

\begin{remmark}
Note that step 2 is very general in theory. In practice, people  typically choose decision tree as the base learner. In our case, we use PMT as the base learner.
\end{remmark}

In the end, we give a definition of  Gateaux derivative of functions between Banach spaces $\mathcal{A}$ and $\mathcal{B}$. Let $U \subset \mathcal{A}$ be open, and $F: U \to \mathcal{B}$. The Gateaux differential $F'(f;h)$ of $F$ at $f\in U$ in the direction $h \in \mathcal{A}$ is defined as:

\[
F'(f;h)=\frac{\partial F(f+\lambda h)}{\partial \lambda}\bigg|_{\lambda=0}=\lim_{\lambda \to 0}\frac{F(f+\lambda h)-F(f)}{\lambda}
\]

The second derivative $F''(f;h)$ is defined similarly.  For example, let $f,\mathbf{x}$ be vectors in an inner product space $\mathcal{V}$. Denote $F(f)=x^{T}f$, then
\[
F'(f;h)=\lim_{\lambda \to 0}\frac{\mathbf{x}^{T}f+\lambda \mathbf{x}^{T}h-\mathbf{x}^{T}f}{\lambda}=\mathbf{x}^{T}h
\]

\subsection{Consistency of SBPMT}

We first give a lemma showing the connection of consistency of individual classifier and the voting classifier. Similar results have been proved in \cite{consisrf}, \cite{GAO2022103788}. 

\newtheorem{lemmma}{Lemma}

\begin{lemmma}
\label{lemma-1}
Suppose the number of subbagged classifiers is $M$ and the sequence of each of them is consistent for the same distribution $\mathcal{P}$, i.e $\{f_{Z_{i},\lfloor \alpha n\rfloor,T}(\mathbf{x})\}_{n=1}^{\infty}$ is consistent for each $i\in \{1,2,...,m\}$, then the voting classifier $f_{M,n,T}(\mathbf{x})$ defined in \eqref{eq-voting} is also consistent.
\end{lemmma}

If the iteration times in Adaboost depends on sample size, \cite{bartlett07b} have shown that Adaboost is consistent if it is stopped at certain step. We restate their result in Theorem \ref{theorem-1}:

\newtheorem{theoremm}{Theorem}
\begin{theoremm}[\cite{bartlett07b}]
   \label{theorem-1}
    Assume $d_{VC}(\mathcal{H})$ is finite and $L^{*}>0$, 
    \[
    \lim_{\lambda \to \infty}\inf_{f\in \mathcal{F}_{\lambda}R(f)}=R^{*}
    \], where $\mathcal{H}$ is the base classifier space, $R^{*}=\inf R(f)$ over all measureable functions,
    $t_{n}\to \infty$, and $t_{n}=O(n^{\nu})$ for $\nu <1$. Then AdaBoost stopped at step $t_{n}$ returns a sequence of classifiers almost surely satisfying $L(g(f_{t_{n}}))\to L^{*}$.
\end{theoremm}

If we can show that the hypothsis space $\mathcal{H}$ in SBPMT has finite VC-dimension, consistency of SBPMT follows directly from Lemma \ref{lemma-1} and Theorem \ref{theorem-1} as long as  we stopped adaboost procedure at the step required by \cite{bartlett07b}. This gives rise to one of our main theorems:

\begin{theoremm}[Consistency of SBPMT]
   \label{theorem-2}
   Let $\mathcal{H}$ be the hypothesis space consists of Probit Model Trees, then $d_{VC}(\mathcal{H})<\infty$ and $E[L(g(f_{M,n,T_{n}}))]\to L^{*}$, where $M$ is the number of subagging classifiers, $n$ is the sample size, $T_{n}=O(n^{\nu})$ for $\nu<1$, i.e the final classifier returned by SBPMT is consistent.
\end{theoremm}

\begin{proof}

Note that $\mathcal{H}$ in SBPMT consists of PMTs  which are decision tress composed with Probit boosting functions, by intuition, the VC-dimension should be finite as well.

Denote the hypothesis space $\mathcal{H}_{1}$ be the set consists of decision trees (with partition structure).  And let $\mathcal{H}_{2}$ consist of classifiers in the form of linear indicator functions generated from ProbitBoost procedure in Algorithm \ref{algo-1} with weighted least square regressors being weak learners. More specifically, $\mathcal{H}_{2}=\{ f| f(\mathbf{x})=sign(\mathbf{\beta}^{T}\mathbf{x}+b),\mathbf{\beta} \in \mathbb{R}^{p},b\in \mathbb{R}\}$ when the dimension of input data is $p$. Then it's easy to check that 

\[
d_{VC}(\mathcal{H}_{2}) =p+1 < \infty
\]

Note that $\mathcal{H}$ in SBPMT consists of PMTs. We  have

\[
d_{VC}(\mathcal{H}) \leq d_{VC}(\mathcal{H}_{1} \cup \mathcal{H}_{2}) \leq d_{VC}(\mathcal{H}_{1})+d_{VC}(\mathcal{H}_{2})+1
\]

\cite{VCdt} has shown that the VC-dimension of a binary tree structure is of order $O(Nlog(Nl))$ where $N$ is the number of internal nodes and $l$ is the number of real-valued features. Once we fix the depth of decision trees $d$, $d_{VC}(\mathcal{H}_{1})$ will be finite. The consistency of SBPMT follows from Theorem \ref{theorem-1}.
\end{proof}
\subsection{Generalization error bound}

Derived from Lemma \ref{lemma-1} and Theorem \ref{theorem-1}, Theorem \ref{theorem-2} implies that the  consistency  of SBPMT  mainly results from the consistency of AdaBoost, which is true only when we connects stopping rule with sample size in a certain way. However, the proof ignores the effect of subagging which can be more useful in real cases. In this section, we will give an explicit expression of the finite sample upper bound of generalization error of SBPMT  which demonstrates the effect of each hyperparameter in SBPMT hence guides us how to improve the performance of SBPMT in practice. We will set the boosting iteration be independent of the sample size since the requirement for consistency proved in Theorem \ref{theorem-2} is not realistic in most real cases 

\begin{theoremm}
 \label{theorem-3}   

Suppose the number of subagged classifiers is $M$ and $M>ln^{2}n$ , then over the random sampling in subagging step, with  $\delta(>0)$ and probability at least $1-\delta$  we have the following upper bound of  generalization error of voting classifier $f_{M,n,T}$ defined in equation \ref{eq-voting}:
\begin{equation}
    \label{ineq_generalization}
        P_{\mathcal{P}}(g(f_{M,n,T}(\mathbf{X}))\neq Y) \leq exp\bigg( \frac{-(\frac{\lceil M/2 \rceil -M/2}{M}+1-2p_{sub}
    )^2}{2Q_{A}^{2}\sigma_{1}^{2}+Q_{B}^{2}\beta/2+(\sqrt{Q_{B}\gamma}+4Q_{C}^{2}/3)(\frac{\lceil M/2 \rceil -M/2}{M}+1-2p_{sub}) } \bigg)
\end{equation}

,where:
\begin{align}
\begin{split}
   p_{sub}&=P_{\mathcal{P}}(f_{Z_i,m,T}(X)\cdot Y\leq 0) ,i=1,2,...,M.\\
m&=\lfloor \alpha n\rfloor,0<\alpha \leq 1\\
\sigma_{1}^{2}&=Var(\mathbb{E}(f_{Z_{i},m,T}(X_{1},...,X_{m})|X_{1}))\\
Q_{A}&= \sqrt{\frac{m^2}{n}}+\bigg(1+4\sqrt{ln(3/\delta)}\bigg)\sqrt{\frac{m}{M}}\\
Q_{B}&=\frac{m^2}{n}+\bigg(1+4\sqrt{ln(3/\delta)}\bigg)\frac{m}{\sqrt{M}}\\
Q_{C} &=\frac{m}{n}+\frac{\sqrt{2m}+3}{\sqrt{M}}ln(3/\delta)  
\end{split}
\end{align}
The subscript $\mathcal{P}$ means that random samples in a given dataset are all generated from the unknown distribution $\mathcal{P}$.
\end{theoremm}
\begin{remmark}
\hfill
\begin{enumerate}
  \item Note that the RHS of the inequality above is a monotonic increasing function w.r.t the value of $p_{sub}$. This enables us to replace $p_{sub}$ by some upper bounds so that we can make the inequality for the generalization error upper bound more flexible. 
  \item We have two layers of randomness in SBPMT. The outer layer is the random sampling from from the unknown distribution $\mathcal{P}$. The inner randomness, which is measured by the positive number $\delta$, comes from the subagging step.
\end{enumerate}
\end{remmark}

Theorem \ref{theorem-3} tells us that the number of subaggingg times  plays a role in controlling the generalization error bound of SBPMT. When we increase $M$, the RHS of the inequality above will decrease accordingly. However, quantities $Q_{A},Q_{B}$ and $Q_{C}$ won't vanish even when $M \to \infty$,  which implies that the large number of $M$ doesn't necessarily give a better performance of the ensemble classifier. Such observation motivates us to improve time efficiency in practice by not setting $M$ too large. The requirement that subagging numbers should be larger than $ln^{2}n$ is not too strict in practice. Suppose we have 20000 observations, Theorem \ref{theorem-3} suggests that we need to set $M$ be larger than $ln^{2}(20000)\approx 98.08$.

On the other hand, the RHS of inequality \eqref{ineq_generalization} is a decreasing function w.r.t $p_{sub}$. That means, hyperparameters benefiting the performance of boosted classifiers also improve the generalization error bound.  Based on this property, we recall two useful theorems in the margin theory of AdaBoost.

\begin{theoremm}[\cite{Boosting_Foundations_and_Algorithms}]
\label{theorem-4}
    Suppose we run AdaBoost for T rounds on $n$ random samples, using base classifiers from a hypothesis space $\mathcal{H}$ with finite VC-dimension $d$. Assume $n\geq max\{d,T\}$, then the boosted classifier $F$ satisfies
    \[
    P_{\mathcal{P}}(YF(X)\leq 0) \leq P_{S}(YF(X)\leq 0)+\sqrt{\frac{32[T(ln(en/T))+d(ln(em/d))+ln(8/\delta)]}{m}}
    \],
with probability at least $1-\delta$ (over the choice of random sample from distribution $\mathcal{P}$). $P_{S}(YF(X)\leq 0)$ is the empirical error rate of boosted classifier $F$.
\begin{remmark}
If we replace $F$ in the LHS of the inequality with a subagging classifier, $P_{\mathcal{P}}(YF(X)\leq 0)$ can be written as $p_{sub}$.
\end{remmark}
\end{theoremm}

\begin{theoremm}[\cite{margin_adaboost}]
\label{theorem-5}
Given a training set $S$, suppose AdaBoost generates classifiers with weighted training errors $\varepsilon_{1},...,\varepsilon_{T}$ and final weighted classifier is $f$. Then for any $\theta $, we have 
\[
P_{S}(Yf(X)\leq \theta)\leq 2^{T}\prod_{t=1}^{T}\sqrt{ \varepsilon_{t}^{1-\theta}(1-\varepsilon_{t})^{1+\theta}}
\]
\end{theoremm}

Now let $\theta =0$, we have

\[
P_{S}(Yf(X)\leq 0)\leq 2^{T}\prod_{t=1}^{T}\sqrt{ \varepsilon_{t}(1-\varepsilon_{t})}
\]

Let $\varepsilon_{t}= 1/2-\gamma_{t}$, we can further obtain the following inequality (\cite{Boosting_Foundations_and_Algorithms}):
\begin{equation}
   P_{S}(Yf(X)\leq 0)\leq \prod_{t}^{T} \sqrt{1-4\gamma_{t}^2} \leq exp\bigg( -2\sum_{t=1}^{T}\gamma_{t}^{2}\bigg) 
   \label{ineq1}
\end{equation}
if $0< \gamma_{t}\leq \frac{1}{2}$ for $t=1,2,...,T$. 

Inequality \eqref{ineq1} shows that the training error of AdaBoost decreases exponentially fast. However, the second term in the RHS of the inequality in Theorem \ref{theorem-4} might increase as well when $T$ becomes large. This observation implies that there might be an optimal iteration time $T^{*}$ for AdaBoost which gives the smallest generalization error upper bound. 

On the other hand, when $T$ gets even larger, the upper bound in Theorem \ref{theorem-4} becomes useless. \cite{margin_adaboost} gave an example showing that the boosting can have even smaller testing error when the training error becomes zero already. Although they also proposed the margin theory to qualitatively explain the shape of the observed learning curves and \cite{GAO20131} proved a sharper upper bound of generalization error of AdaBoost based on margin theory, it's still unclear how these results provide practical guidance in controlling the prediction error. 

Fortunately, we are able to give a specific upper bound like Theorem \ref{theorem-4} by properties of ProbitBoost, which is useful in controlling the generalization error by not setting  $T$ too large. Lemma \ref{lemma-2} and Lemma \ref{lemma-3} are the two most critical properties of ProbitBoost we will use.

\begin{lemmma}
\label{lemma-2}
Let $Q(x)=-log(\Phi(x))$ where $x$ is any real number and $\Phi$ is the cumulative probability function for standard normal distribution. Then $Q''(x)>0$ and $Q'''(x)<0$.
\end{lemmma}

\begin{lemmma}
\label{lemma-3}
    Let $P_{n}$ be the empirical probit risk function defined in section 3.1. A sequence of functions $\{f^{t}\}_{t\geq 0}$ is generated by ProbitBoost described in section 3.1-3.2, i.e $f^{t+1}=f^{t}+g^{t}$ where $g^{t}$ is the base learner fitted at $t$-th iteration of ProbitBoost. Then for a given dataset and  each $t\geq 0$,
    \[
    P_{n}(f^{t})-P_{n}(f^{t+1})\geq \frac{1}{2}\gamma ||g^{t}||_{2}^2
    \].
    In particular, $P_{n}(f^{t}) \downarrow \inf_{k}P_{n}(f^{k})$ as $t\to \infty$, $\sum_{t=1}^{\infty}||g^{t}||_{2}^{2}<\infty$, and $\lim_{t \to \infty}||g^{t}||_{2}=0$, where $||g^{t}||_{2}=\sqrt{\sum_{i=1}^{n}g^{t}(x_i)^2/n}$ and $\gamma=\inf\{P_{n}''(f),f\in \mathcal{H} \}$ ( which is positive for a given data set)
\end{lemmma}
In Lemma \ref{lemma-3}, we assume that each observation is equally weighted when we calculate empirical probit risk function. As we can see in the proof of Lemma \ref{lemma-3}, the non-increasing property of empirical probit risk function in ProbitBoost is still true for arbitrary weighted distribution of sample points due to the definition of Gateaux derivative. This observation is critical in seeing the importance of Theorem \ref{theorem-6}. We now give Corollary 1 without proof to emphasize this point:
\newtheorem{corollaryy}{Corollary}
\begin{corollaryy}
\label{col-1}
    The non-increasing property still holds when we apply the ProbitBoost to minimize the weighted probit risk function $P_{n,\mathbf{w}}$, where $\mathbf{w}$ is the weight vector assigned to sample points.
\end{corollaryy}
Lemma \ref{lemma-3} implies that increasing the number of iterations in ProbitBoost will always reduce the value of empirical probit risk function.  Note that for $\forall z \in R$:
\begin{equation}
\label{eq17}
    I(z\leq 0) \leq \frac{Q(z)}{ln2}
\end{equation}
It follows that large number of iterations in ProbitBoost at each terminal node of PMT will generate a smaller upper bound of weighted training error. As a result, we can iterate Adaboost procedure in SBPMT for a few times to achieve good performance with the price of relatively large number of iterations in ProbitBoost. As we pointed out in section 3.2, generating a simple linear combination of features makes a ProbitBoost classifier maintain the interpretability even after a long time iteration. Thus the final model won't be too complicated. Theorem \ref{theorem-6} then gives the upper bound of generalization error of AdaBoost when PMTs are base learners:
\begin{theoremm}
\label{theorem-6}
    Suppose we run AdaBoost for T rounds on $n$ random samples, using base classifiers from a hypothesis space consisting of PMTs. Let $\varepsilon_{t}^{P}(B)$ be the weighted probit risk of a PMT at $t$-th iteration of AdaBoost. Integer number $B$ represents the ProbitBoost iteration times at each terminal node of a PMT. Let $d$ be the  upper bound of VC-dimension of the hypothesis space consisting of PMTs.
    Assume $n\geq max\{d,T\}, 0\leq \varepsilon_{t}^{P}(B)/ln(2) < \frac{1}{2}$, then the boosted classifier $F$ satisfies
    \begin{equation}
    \label{eq16}
            P_{\mathcal{P}}(YF(X)\leq 0) \leq exp\bigg(-2\sum_{t=1}^{T}\gamma_{t}^{2}(B)\bigg)+\sqrt{\frac{32[T(ln(en/T))+d(ln(en/d))+ln(8/\delta)]}{n}}
    \end{equation},
with probability at least $1-\delta$ (over the choice of random sample from distribution $\mathcal{P}$). 
$\gamma_{t}(B)$ is defined to be:
\[\gamma_{t}(B)=\frac{1}{2}-\frac{\varepsilon_{t}^{P}(B)}{ln(2)}\].
\end{theoremm}

According to Corollary \ref{col-1}, for fixed $T$, if we fit weighted probit risk function at each partitioned space, $\varepsilon_{t}(B)$ will decrease as we increase the ProbitBoost iteration time $B$, which in fact lowers down the value of $exp(-2\sum_{t=1}^{T}\gamma_{t}^{2}(B))$ as long as we set $B$ large enough. That means the assumption of $\varepsilon_{t}^{P}(B)$ in Theorem \ref{theorem-6} makes sense in practice.

The main message conveyed by Theorem \ref{theorem-6} is that employing PMTs as base learner in AdaBoost can not only lead to a smaller prediction error without setting running time $T$ of AdaBoost large but also make the second term in the upper bound \eqref{eq16} meaningful when $T$ is small. The cost is we need to set large number of ProbitBoost iterations at each terminal node of trees to achieve smaller (weighted) training error for each partition. Combining these facts with Theorem \ref{theorem-3}, we anticipate that SBPMT with less AdaBoost iterations steps and subagging times could have even better performance than ensemble methods using ordinary (CART) trees (no model included) as base learner. In section 6, the performance of SBPMT will be compared with other popular ensemble learning algorithms.
\section{Experiments}

\subsection{Simulation}

To verify the conjecture in section 5, we created an artificial dataset which is in the spirit to the work of \cite{Simudata}. More specifically, we will generate data point from the following model:

\[
P(Y=1|\mathbf{x})=q+(1-2q)\mathbf{1}\bigg( \sum_{e=1}^{E}\mathbf{x}^{(e)}>E/2\bigg)
\]
where $\mathbf{x}^{e}$ is the $e$-th dimension of an observation $\mathbf{x}$. The response $\mathbf{Y}\in \{-1,1\}$ and random sample $X$ will be i.i.d and uniformly distributed on the $d$-dimension unit cube $[0,1]^{d}$. According to \cite{Simudata}, $q$ is the Bayes error and $E\leq d$ is the effective dimension. For time efficiency, we only set $d=10$ and effective dimension $E=5$. There are 2000 training cases and 10000 test observations. Each subplot in Figure \ref{fig2} illustrates test error rate for SBPMT, as a function of a single hyperparameter. The benchmark setting for SBPMT is $M=5,B=5,T=5$ and $\alpha=0.7$. And each CART tree has depth $d=3$ with min\_leaf\_size  20. To obtain, for example the test error as a function of subagging times, we set $M=1,2,...,100$ while other parameters remain the same with benchmark setting. Similar strategy was applied to the other three hyperparameters we are interested.

\begin{figure}[htb]

 \begin{subfigure}{0.48\textwidth}
     \includegraphics[width=\textwidth]{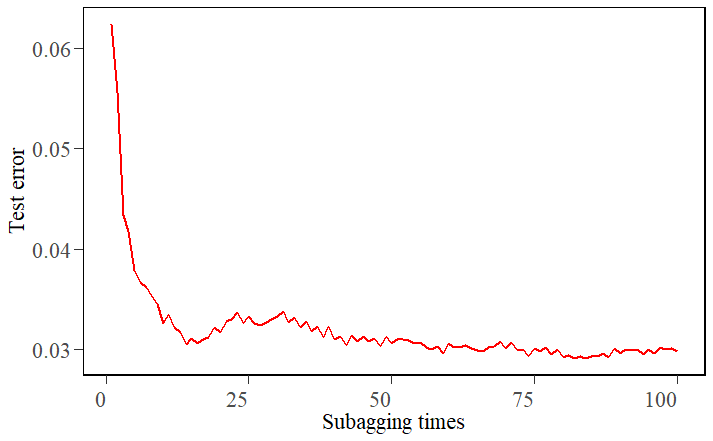}
     \caption{}
     \label{fig:a}
 \end{subfigure}
 \begin{subfigure}{0.48\textwidth}
     \includegraphics[width=\textwidth]{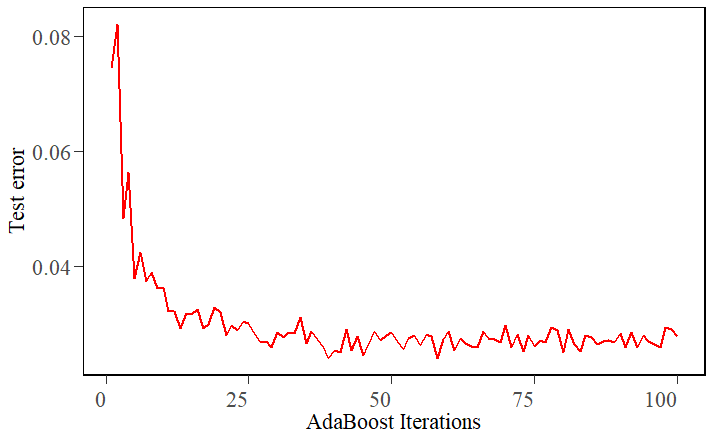}
     \caption{}
     \label{fig:b}
 \end{subfigure}

 \begin{subfigure}{0.48\textwidth}
     \includegraphics[width=\textwidth]{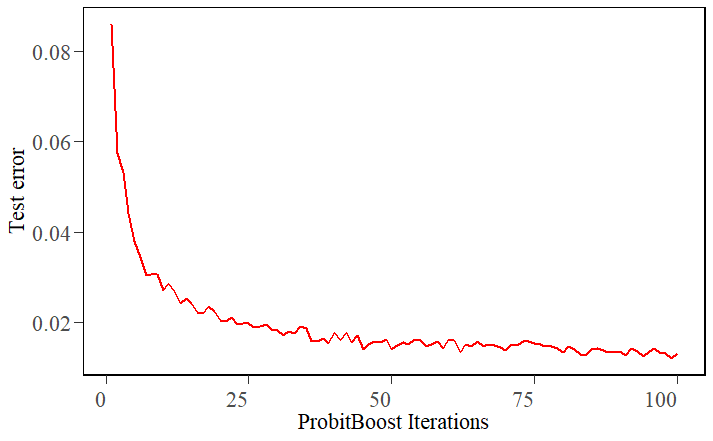}
     \caption{}
     \label{fig:c}
 \end{subfigure}
 \begin{subfigure}{0.48\textwidth}
     \includegraphics[width=\textwidth]{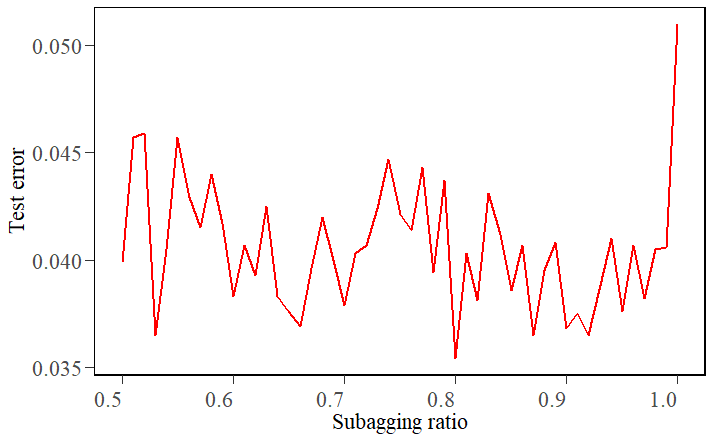}
     \caption{}
     \label{fig:d}
 \end{subfigure}

 \caption{Test errors for SBPMT. (a) Test error rate as a function of Subbagging Times, (b) Test error rate as a function of AdaBoost iterations, (c) Test error rate as a function of ProbitBoost iterations, (d) Test error rate as a function of Subagging ratio.}
 \label{fig2}
\end{figure}

As we can see in Figure 2, as subagging times increase the test error rate will in general decreases. When the subagging time is large, the improvement of test prediction is negligible which obeys the anticipation from inequality (\ref{ineq_generalization}). On the other hand, when boosting iterations for AdaBoost part proceed, the error rate  decreases as well. Similar observation can be made for boosting iterations for ProbitBoost part. Nevertheless, it can be seen in subplots that test error w.r.t ProbitBoost iterations decreases more steadily which implies a smaller variance while it shows a relatively large volatility in AdaBoost iterations. Such difference can be inferred from fact that the empirical probit risk function is tigher than exponential risk w.r.t the indicator function, which favors the non-increasing property in Lemma \ref{lemma-3}. What's more, test error function of ProbitBoost iterations has the lowest long run test error. This is a clue showing that ProbitBoost can dominate the effect from subagging and AdaBoost. In other words, local structure can indeed capture data patterns which may be ignored by global structure of the ensemble method. It's hard to see any pattern for test error and  subagging ratio. This observation can also expected from the generalization error bound in Theorem \ref{theorem-3} as $\alpha$ appears in not only three quantities $Q_A,Q_B$ and $Q_C$ but also $p_{sub}$ which is far more complicated. The only thing we can clearly see is, when we set $\alpha=1$, i.e the full dataset is utilized, we will end with relatively high test error as SBPMT will reduce to a single boosted classifier which might learn noise from the whole dataset.

\subsection{Performance with real datasets}
\setlength{\tabcolsep}{6pt} % Default value: 6pt
\begin{table}[htb]
\scriptsize
\captionsetup{font=footnotesize,singlelinecheck=off}
\caption{ Real datasets used for the experiments,sorted by size} % title name of the table  
\centering % centering table  
\begin{tabular}{ccccc} % creating 10 columns  
\hline
Dataset &Instances&  Numerical attributes  & Categorical attributes & Classes 
\\ [0.5ex]  
\hline
% Entering 1st row  
Iris & 150&4& 0&3 \\
  \\ [-1ex] 
Glass & 214&9& 0&6 \\
  \\ [-1ex]  
Ionosphere& 351&33 & 0&2 \\
 \\ [-1ex]  
 Diabetes&520&0& 17&2\\
         \\ [-1ex]  
 Breast-Cancer& 569&30 & 0&2 \\
 \\ [-1ex]  
   Balance-scale & 625&4 &0 &3\\
     \\ [-1ex]  
 Australian  & 690&6 & 8&2 \\
 \\ [-1ex]  
Pima-indians &768&8 & 0&2 \\
 \\ [-1ex]  
   Vehicle & 846&18 & 0&4\\
     \\ [-1ex]  
     Raisin & 900 &7& 0&2\\
   \\ [-1ex]  
    Tic-tac-toe &958&0 & 9&2\\
        \\ [-1ex] 
 German &1000&6 & 14&2\\
      \\ [-1ex] 
 Biodegradation & 1051&38 & 3&2 \\
     \\ [-1ex]  
 BHP&1075&20& 1&4\\
         \\ [-1ex]  
      Diabetic & 1151 &16 & 3&2\\
 \\ [-1ex] 
 Banknote &1372&4 & 0&2\\
      \\ [-1ex]   
  Contraceptive&1473&2 & 7&3\\
         \\ [-1ex]  
Obesity&2111&3& 13&7\\
         \\ [-1ex]  
Segments&2310&19& 0&7\\
         \\ [-1ex]  
Waveform+noise&5000&40& 0&3\\
         \\ [-1ex]  
Pendigits&10000&16& 0&10\\
         \\ [-1ex]  
Letter&20000&16& 0&26\\
% [1ex] adds vertical space  
\\[-1ex]
\hline
\end{tabular}  
\label{table-1}
\end{table}

\setlength{\tabcolsep}{4pt} % Default value: 6pt

\begin{table}[htb]
\scriptsize
\captionsetup{font=footnotesize,singlelinecheck=off}
\caption{ Mean classification accuracy and standard deviation for SBPMT vs.RandomForest,GradientBoost(100)\\AdaBoost.M1(100),XGBoost(10),XGBoost(100)} % title name of the table  
\centering % centering table  
\begin{tabular}{cccccccc } % creating 10 columns  
\hline
Dataset &SBPMT  & RandomForest  &GradientBoost(100) & AdaBoost.M1(100) & XGBoost(10) &XGBoost(100)
\\ [0.5ex]  
\hline
Iris &\textbf{96.00 $\pm$ 5.62} & 94.67$\pm$ 5.26 &94.67$\pm $6.13 & 94.00$\pm$ 7.34 &95.33 $\pm$ 5.49&95.33$\pm$5.49\\
  \\ [-1ex]  
Glass & 75.67 $\pm$ 8.97 & 77.09$\pm$ 9.57 &75.27$\pm $9.55 & 73.26$\pm$9.60&75.62 $\pm$ 8.78&77.53 $\pm$9.35\\
  \\ [-1ex]  
Ionosphere& 92.87 $\pm$ 2.75 & 93.68$\pm$ 3.84& 92.00 $\pm$ 3.77& 93.68$\pm$ 3.39&92.86 $\pm$ 2.79&92.85 $\pm$2.10\\
 \\ [-1ex]  
 Diabetes&95.96 $\pm$ 1.42 & 97.69$\pm$ 1.52&91.35$\pm $3.04& 97.5$\pm$2.04 &95.38$\pm$ 2.26&95.38 $\pm$2.26\\
       \\ [-1ex] 
       Breast-Cancer & \textbf{97.03 $\pm$ 2.45} &96.34 $\pm$ 4.27& 95.82 $\pm$4.62& 96.16 $\pm$ 3.62&93.89 $\pm$5.72 &95.64 $\pm$ 4.96\\
 \\ [-1ex] 
 Balance-scale &  \textbf{95.19 $\pm$ 2.62}  & 83.53 $\pm$ 4.08& 90.88 $\pm$1.52 & 85.46$\pm$ 3.86&86.56$\pm$ 2.73 &87.84 $\pm$3.22\\
     \\ [-1ex] 
     Australian  & 86.39 $\pm$ 1.97 & 87.41$\pm$ 3.69& 86.25$\pm$ 3.11 & 85.26$\pm$ 4.90 &85.53$\pm$ 3.81&86.97 $\pm$3.03\\
 \\ [-1ex] 
Pima-indians & \textbf{77.73 $\pm$ 6.24} & 76.57$\pm$ 6.12&75.78 $\pm$ 5.30 & 73.31$\pm$ 6.03&75.26 $\pm$ 7.07&75.65 $\pm$6.76 \\
 \\ [-1ex]  
   Vehicle &\textbf{82.97 $\pm$ 4.92}& 75.30 $\pm$ 3.95&72.81 $\pm$3.43 &78.49 $\pm$ 3.43&76.35 $\pm$ 4.36 &78.25 $\pm$ 3.90\\
     \\ [-1ex]  
Raisin & 86.44 $\pm$ 3.70 & 86.44 $\pm$ 4.02& 86.56 $\pm$3.65 & 86.89 $\pm$ 4.68&84.67$\pm$ 6.01 &84.33 $\pm$ 4.04\\
   \\ [-1ex]  
      Tic-tac-toe & 97.91 $\pm$1.56 &99.16 $\pm$ 1.08&74.23 $\pm $3.76  & 99.58 $\pm$ 0.73&98.22 $\pm$ 1.72 &98.33 $\pm$1.50\\
     \\ [-1ex]  
 German &74.80 $\pm$ 3.39 & 76.20$\pm$ 3.74&75.20 $\pm$  2.78 & 74.60$\pm$3.17 &74.30 $\pm$ 3.34&75.7 $\pm$2.41\\
      \\ [-1ex] 
       Biodegradation & \textbf{87.59$\pm$ 1.92 }& 87.20 $\pm$ 2.47&  84.16 $\pm$3.73 & 86.54 $\pm$1.64&86.16 $\pm$ 2.67 &86.63 $\pm$ 1.84\\
     \\ [-1ex]  
BHP&100 $\pm$ 0 & 100$\pm$ 0.00&86.90$\pm $3.79  &100$\pm$0.00&98.97 $\pm$ 1.36&100 $\pm$0.00\\
       \\ [-1ex]  
Diabetic & \textbf{74.02 $\pm$ 4.96} &67.95$\pm$ 5.68& 65.86 $\pm$  4.56   & 67.95$\pm$ 4.12 &68.98 $\pm$ 4.89&67.86 $\pm$5.33\\
 \\ [-1ex]  
   Banknote &\textbf{99.78 $\pm$ 0.35} & 99.34$\pm$ 0.94&96.72 $\pm $2.02&99.64 $\pm$ 0.52&98.61 $\pm$ 1.44 &98.69 $\pm$ 1.32\\
     \\ [-1ex]  
  Contraceptive&  \textbf{55.93 $\pm$ 4.42}  & 54.16$\pm$ 3.33&55.46$\pm $2.56 & 55.73$\pm$3.33 &54.44 $\pm$ 3.34&52.00 $\pm$3.93\\
       \\ [-1ex]   
Obesity&97.39 $\pm$1.10& 95.03$\pm$  1.12&82.19$\pm $2.18 &97.44 $\pm$  1.03&95.08$\pm$ 1.45&97.49 $\pm$0.92\\
       \\ [-1ex]  
Segments&98.31 $\pm$ 1.18 & 97.79$\pm$ 1.84&94.76$\pm $2.04  & 98.53$\pm$1.24 &97.71 $\pm$ 1.61&98.48 $\pm$1.25\\
       \\ [-1ex]  
Waveform+noise&\textbf{86.16 $\pm$ 1.39} & 85.36$\pm$ 2.08&84.78$\pm $1.93 & 84.56$\pm$2.02 &84.26$\pm$ 1.74&85.38 $\pm$1.93\\
       \\ [-1ex]  
Pendigits& \textbf{ 99.38 $\pm$ 0.24} & 99.21$\pm$ 0.25&90.63$\pm $1.17 & 98.57$\pm$0.25 &98.20 $\pm$ 0.49&99.18 $\pm$0.33\\
       \\ [-1ex]  
Letter&95.50 $\pm$ 0.31 & 96.85$\pm$ 0.24&73.17$\pm $0.64  & 96.12$\pm$0.36 &90.14 $\pm$ 0.90&96.60 $\pm$0.27\\
% [1ex] adds vertical space  
\\[-1ex]
\hline
\end{tabular}  
\begin{tablenotes}\scriptsize
\item[*]  *The best performing values are highlighted in bold
\end{tablenotes}
\label{table-2}
\end{table}  

 We used 22 benchmark datasets from the UCI repository (\citeauthor{UCI}) to evaluate the performance of SBPMT and compare against with other state-of-the-art boosting learning algorithms. The details of those dataset are included in Table \ref{table-1}. For simplicity, we didn't use datasets with missing values but there are many methods to deal with this issue in practice. The size of datasets ranges from one hundred to twenty thousand observations. They contain varying number of numerical attributes and categorical attributes. We also deal with some multi-class problems. 

 Ten fold stratified cross-valiation was performed for each dataset and algorithm.  For each algorithm, they shared the same training and testing splitting. Mean classification accuracy and standard deviation are reported in Table \ref{table-2}\footnote{The default subsampling ratio in  GradientBoost and XGBoost is set to be 0.7 for the sake of comparison}. All algorithms are implemented by corresponding R packages.\footnote{We employed "rpart" package for RandomForest;R package "adabag" and "RWeka“ are used for AdaBoost; "xgboost" is used for XGBoost algorithm; "gbm" is used for GradientBoost. R markdown files for experiments and figures can be found in https://github.com/BBojack/SBPMT}

In practice, people can apply cross-validations to determine the best value of hyperparameters. In theory, for instance, theorem \ref{theorem-3}  suggests that the subagging number $M$  be larger than $ln^{2}(n)$. Lemma \ref{lemma-2} and Theorem \ref{theorem-5}  encourage small number iterations in AdaBoost  but  large number of iterations in ProbitBoost when using PMTs as base learners.  For simplicity and time efficiency, we used same hyperparametes in SBPMT for all datasets. More specifically, number of subagging $M=21$, number of AdaBoost iteration $T=5$, number of ProbitBoost iteration $B=100$, subagging ratio $\alpha=0.7$, depth of CART tree $l=6$ and min\_leaf\_size  is 20. 
We set relatively small subagging times in order to speed up the training part. It turns out that small value of $M$ is enough for SBPMT achieving decent accuracy in many real cases. If time efficiency is not a priority, we recommend that the value of $M$ should satisfy the theoretical inequality in Theorem \ref{theorem-3}. Otherwise, a small value of $M$ is enough. The number of AdaBoost iteration and the number of ProbitBoost iteration are set in accordance with analysis in section 5 and simulation results.

 From Table \ref{table-2} we see that SBPMT reaches  almost the same mean classification accuracy as XGBoost(with 100 iterations) and  many other popular ensemble algorithms. It outperforms all selected algorithms on 11 datasets which consist of small, medium and large-sized datasets. That implies SBPMT is competitive with any size of datasets. Compared with Random Forest which typically needs 500 subtrees, SBPMT in our experiments only uses 21 subagged samples. More over, the AdaBoost iteration times in SBPMT is only 5 which is much smaller than that of normal AdaBoost using hundreds of CART trees as base learners. Based on these points, SBPMT makes the final model potentially easier for practitioner to understand without digging into too many subtrees. The linear form of PMT at terminal nodes of trees enables people to interpret the effect of features at each partitioned space. 
%The reason why we set ProbitBoost iteration large follows from the observation in section 5 that larger $\gamma_{PMT}$ may give rise to smaller $p_{sub}$ which further makes SBPMT have lower generalization error.

In Figure \ref{fig-trend-acc}, we first sorted the datasets by their size. and plotted the accuracy of SBPMT, RF and XGboost(100). When the size of a dataset is small of medium, we observe that SBPMT outperforms RF and XGboost(100). When we have large number of observations, these three algorithm have similar performance. This phenomenon provides an evidence that SBPMT is generally good and might be especially powerful in smaller datasets. 

In Figure \ref{fig-trend-ratio}, we  sorted the datasets by the ratio of the number of numerical attributes and the total number of attributese. If some datasets have a same ratio, they were sorted by the size. Accuracies of SBPMT, RF and XGboost(100) were given in y-axis. We observe that when features in a dataset are all numerical, SBPMT is likely to outperform than other methods. One possible explanation is, we fit a linear model by ProbitBoost at each terminal node of a CART tree, which would favor numerical attributes. Despite of the complexity of real world datasets, our results provide an empirical clue that SBPMT may perform better in datasets with more numerical attributes.

\begin{figure}
\includegraphics[width=0.95\textwidth]{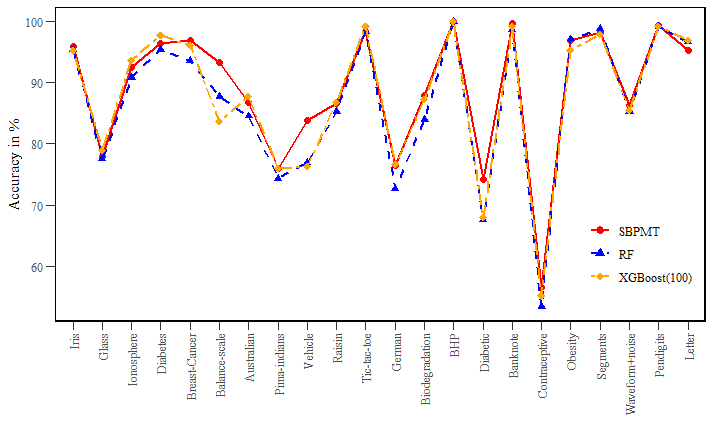}
\caption{Accuracy of SBPMT,RF and XGboost(100). Datasets are sorted by size in ascending order (left to right)}
\label{fig-trend-acc}
\end{figure}
\begin{figure}
\includegraphics[width=0.95\textwidth]{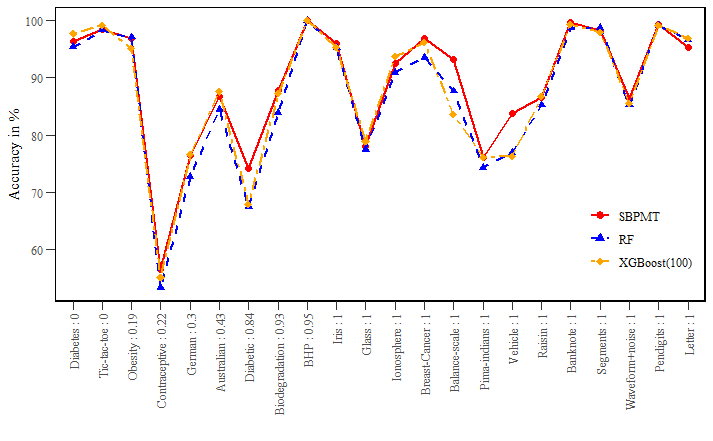}
\caption{Accuracy of SBPMT,RF and XGboost(100) w.r.t the ratio of the number of numerical attributes and the total number of attributes. Datasets are sorted by size in ascending order if they have a same ratio.}
\label{fig-trend-ratio}
\end{figure}

\newpage
\section{Discussion}

For now, we have shown the consistency of SBPMT when the iteration number in AdaBoost part is a function of sample size. Nevertheless, this condition is not realistic in many cases. For large sample size, it's time consuming to run too many rounds of AdaBoost. It's possible to think about a more general condition for the consistency of SBPMT. That might be one of future directions. At the same time, we provides a useful upper bound of the generalization error of SBPMT by using a latest exponential inequality of incomplete U-statistics (\cite{maurer2022exponential}). Inequality in Theorem \ref{theorem-3} explains how subagging helps to improve the performance of ensemble methods. Non-increasing property of the empirical probit risk function under ProbitBoost suggests us use large number of boosting iteration for ProbitBoost in PMT. Combining this result with properties of AdaBoost, we only need to use a relative small  number of AdaBoost iterations in practice to achieve a good performance of SBPMT. In section 6 we compared the performance of SBPMT with real datasets. The results demonstrate that SBPMT is competitive with many popular ensemble learning algorithms, especially for datasets with small and medium sample size. For datasets whose features are all numerical, there is evidence that SBPMT outperforms other methods. That would be an interesting direction for furture research. 

Because of the flexity of bagging-boosting type algorithm, we can propose many other variants of SBPMT. In this paper, we construct a PMT based CART structure which is simple but sufficient. In literature, there are other methods to build a decision tree. For example, it's not necessary to split a feature space by axis-parallel method. The other way is oblique splitting, which means we cut the feature space by the linear combination of features.In the groundwork of Classification And Regression Trees(CART) (\cite{Breiman1984ClassificationAR}), Breiman has already pointed out the idea of splitting the dataset by the combination of variables. However, the implementation of the algorithm in finding the optimal hyperplane in feature space has been the real issue for many years especially for cases involved large dataset. For high-dimensional case, \cite{Breiman1984ClassificationAR}  provided a heuristic algorithm CART-LC to find the relatively good hyperplane. In brief, his method weeds out the unnecessary variable by a backward deletion process. The least important variable is the one whose deletion gives the minimal decrease in impurity. 

%Since we need to repeat the selection process based on whole dataset until no more deletion is required, the computation cost may still be high for large dataset.  The name of the algorithm is CART-LC.

Since the traditional search of optimal hyperplane is likely to stuck in local optimal points , many stochastic optimization goes to the stage. \cite{SADT} introduced Simulated Annealing Decision Tree (SADT) which uses randomization to avoid getting trapped in local optimal value of change of impurity function. The drawback is, the algorithm may evaluate tens of candidates in order to find the global solution. The use of statistic models hasalso been explored. \cite{Truong2009FastGA} used logistic regression to build oblique trees.  \cite{LOPEZCHAU20136283}  induced oblique trees by Fisher's  linear discriminant method. \cite{WICKRAMARACHCHI201612} applied a geometric method in linear algebra called Householder transformation in finding optimal oblique splitting. 

Recently, many other strategies have been taken to find optimal oblique splitting. For example, \cite{RN52}
 formulated an integer programming method to optimize decision trees with fixed depth. \cite{BLANQUERO2020255} proposed a continuous optimization approach to build sparse optimal oblique trees. 
How to employ oblique decision trees with our proposed subagging-boosting type algorithm is another potential direction of future research.

\long\def\acks#1{\vskip 0.3in\noindent{\large\bf Acknowledgments }\vskip 0.2in
\noindent #1}
\acks{This  work was supported  by Lee fellowship in Lehigh University.}

\appendix
\section*{Appendix A. Lemma \ref{lemma-1} proof}

\begin{proof}

Note that 
\begin{equation}
\label{eq19}
        L(f_{m,n})=1-E[\eta(X)I(f_{m}(X)=1)+(1-\eta(X))I(f_{m}(X)=-1)]
\end{equation}

\begin{equation}
\label{eq20}
        L(g)=1-E[\eta(X)I(g(X)=1)+(1-\eta(X))I(g(X)=-1)]
\end{equation}
where $I(\cdot)$ is the indicator function.

Let equation \eqref{eq19} minus equation \eqref{eq20}, and we obtain 

\begin{equation}
\label{eq21}
    \begin{split}
    E[L(f_{m,n})]-L(g^{*})&=L(f_{m,n})-L^{*}=E\bigg[  (2\eta(X)-1)(I(f_{m}(X)=-1)- I(g^{*}(X)=-1)) \bigg]\\
    &=E_{(X,Y)\sim \mathcal{P}}\bigg[E_{\mathcal{Z}}\bigg[  \bigg|2\eta(X)-1\bigg|\bigg(I(f_{m}(X)\neq g^{*}(X) \bigg)\bigg]  \bigg]\\
    &=E_{(X,Y)\sim \mathcal{P}}\bigg[ (1-2\eta(x))I(\eta(x)<1/2)P_{\mathcal{Z}}(f_{m,n}(x)=1)\\
    &+ (2\eta(x)-1)I(\eta(x)>1/2)P_{\mathcal{Z}}(f_{m,n}(x)=-1)\bigg]
    \end{split}
\end{equation}

WLOG, we can assume  $\eta(x)>\frac{1}{2}$ for $\mu$-almost all $x$ where $\mu$ is the probability measure w.r.t distribution $\mathcal{P}$. Now equation \eqref{eq21} becomes

\begin{equation}
    E[L(f_{m,n})]-L(g^{*})=E_{(X,Y)\sim \mathcal{P}}\bigg[  (2\eta(X)-1)P_{\mathcal{Z}}(f_{m,n}(X)=-1)\bigg]
\end{equation}

Similarly, we have 

\begin{equation}
    E[L(f_{S_n,Z_{i}})]-L(g^{*})=E_{(X,Y)\sim \mathcal{P}}\bigg[  (2\eta(X)-1)P_{\mathcal{Z}}(f_{S_n,Z_i}(X)=-1)\bigg]
\end{equation}

By assumption, $EL(f_{S_n,Z_{i}}) \to L(g^{*})$. Thus we must have $P_{\mathcal{Z}}(f_{S_n,Z_i}(X)=-1)\to 0$.

For given $x$ s.t $\eta(x)>\frac{1}{2}$, we have
\begin{equation}
    \begin{split}
        P_{Z}(f_{m,n}(X)=-1)&=P_{\mathcal{Z}}\bigg( \sum_{i=1}^{m}I(f_{S_n,Z_{i}}(X)=-1)\geq \frac{m}{2}  \bigg)\\
        &\leq \frac{m}{2}\sum_{i=1}^{m}E_{Z_{i}}[I(f_{S_n,Z_{i}}(X)=-1)]\\
        &=2P_{\mathcal{Z}}(f_{S_n,Z_{i}}(X)=-1) \to 0
    \end{split}
\end{equation}

It follows that $EL(f_{m,n}) \to L(g^{*})$ as well, i.e the voting classifier is consistent.

\end{proof}

\section*{Appendix B. Lemma \ref{lemma-2} proof}

\begin{proof}
    To show $Q''(x)>0$, we need to perform some calculations directly.  Note that
    \begin{equation}
        Q''(x)=\frac{\phi(x)[x\Phi(x)+\phi(x)]}{\Phi^{2}(x)}
    \end{equation}
    where $\phi(x)$ is the normal density function.

    Recall the fact that $x\Phi(x)+\phi(x)>max(0,x)$ (a simple calculus exercise). It follows that $Q''(x)>0$.

On the other hand, we have:

    \begin{equation}
      Q'''(x)=\frac{\phi(x)[-(x\Phi(x)+2\phi(x))(x\Phi(x)+\phi(x))+\Phi^{2}(x)]}{\Phi^{3}(x)}
    \end{equation}

    It suffices to show $-(x\Phi(x)+2\phi(x))(x\Phi(x)+\phi(x))+\Phi^{2}(x)<0,\forall x \in R$.

Now denote  $-(x\Phi(x)+2\phi(x))(x\Phi(x)+\phi(x))+\Phi^{2}(x)$ as $M(x)$. When $x\geq 0$, we have

 \begin{equation}
     M'(x)=-x^2 \phi(x)\Phi(x)-2x\Phi^{2}(x)-x\phi^{2}(x)-\phi(x)\Phi(x) \leq 0 \quad \text{when $x\geq 0$} 
 \end{equation}
 But $M(0)=-2+\frac{1}{4}=-\frac{7}{4} <0$, we conclude that $M(x)<0$ when $x\geq 0$.

 When $x$ is positive, we will resort to few useful tail inequalities of normal distribution functions. By \cite{sampford} and \cite{SZAREK1999193}, we have

  \begin{equation}
     1-\Phi(x)<\frac{4\phi(x)}{\sqrt{8+x^2}+3x}, \quad \text{for $x\geq 0$}
 \end{equation}

 When $x<0$, we can replace $x$ in (24) by $-x$, which gives us:

   \begin{equation}
   \label{eq29}
     \Phi(x)<\frac{4\phi(x)}{\sqrt{8+x^2}-3x}
 \end{equation}

Next, \eqref{eq29} also implies

\[
x\Phi(x)+\phi(x) > \phi(x)\bigg(\frac{4x}{\sqrt{8+x^2}-3x}+1\bigg)= \phi(x)\bigg(\frac{x+\sqrt{8+x^2}}{\sqrt{8+x^2}-3x}\bigg)
\]
and 
\[
x\Phi(x)+2\phi(x) > \phi(x)\bigg(\frac{4x}{\sqrt{8+x^2}-3x}+2\bigg)= \phi(x)\bigg(\frac{-2x+2\sqrt{8+x^2}}{\sqrt{8+x^2}-3x}\bigg)
\]
It follows that

\begin{equation}
    \begin{split}
        -\bigg(x\Phi(x)+2\phi(x)\bigg)\bigg(x\Phi(x)+\phi(x)\bigg)+\Phi^{2}(x) &<- \phi^{2}(x)\bigg( \frac{2(x+\sqrt{x^2+8})(\sqrt{x^2+8}-x)}{(\sqrt{2+x^2}-x)^2} \bigg)\\
        &+\frac{16\phi^{2}(x)}{(\sqrt{8+x^2}-3x)^2}\\
        &=- \phi^{2}(x)\bigg(  \frac{2(x^2+8-x^2)-16}{(\sqrt{2+x^2}-x)^2}\bigg)=0, \quad \text{for $x<0$}
    \end{split}
\end{equation}

Thus $Q'''(x)<0, \forall x \in R$.

\end{proof}

\section*{Appendix C. Lemma \ref{lemma-3} proof}

\begin{proof}
    By Taylor expansion of  $P_{n}(f^{t})$ in direction of $g^{t}$, we have 
    \begin{equation}
        \begin{split}
            P_{n}(f^{t}) &=P_{n}(f^{t}+g^{t})-P_{n}'(f^{t}+g^{t});g^{t})+\frac{1}{2}P_{n}''(\Tilde{f}^{t+1};g^{t}) \quad \text{(where $\Tilde{f}^{t+1}=f^{t}+\alpha_{t}g^{t}$, $\alpha_{t} \in (0,1)$ )} \\
            &=P_{n}(f^{t+1})-P_{n}'(f^{t}+g^{t})\cdot g^{t}+\frac{1}{2}P_{n}''(\Tilde{f}^{t})\cdot (g^{t})^{2} \quad \text{(where "$\cdot$" represents inner product)}\\
            &=P_{n}(f^{t+1})-\bigg(P_{n}'(f^{t}+g^{t}) -P_{n}'(f^{t})+P_{n}'(f^{t}) \bigg)\cdot g^{t}+\frac{1}{2}P_{n}''(\Tilde{f}^{t})\cdot (g^{t})^{2} 
        \end{split}
    \end{equation}

Perform Taylor expansion of $P_{n}'(F_{t}+f_{t})$ one more time, and we obtain

      \begin{equation}
        \begin{split}
            P_{n}'(f^{t}+g^{t})&=P_{n}'(f^{t})+P_{n}''(f^{t})\cdot g^{t}+\frac{1}{2}P_{n}'''(\hat{f}^{t+1})\cdot (g^{t})^{2} \quad \text{(where $\hat{f}^{t+1}=f^{t}+\alpha_{t}'g^{t}$, $\alpha_{t}' \in (0,1)$ )} \\
        \end{split}
    \end{equation}
By Lemma \ref{lemma-2}, we have

\[
 P_{n}'(f^{t+1})<P_{n}'(f^{t})+P_{n}''(f^{t})\cdot g^{t}
\]

Under the notations of WLS in section 3.2, we can derive the following equation by the property of WLS:

\[
\sum_{i=1}^{n}w_{i}(\hat{a}_{s(t)}x_{s(t),i}+\hat{b}_{s(t)})(z_{i}-\hat{a}_{s(t)}x_{s(t),i}-\hat{b}_{s(t)})=0
\]
which is equivalent to the following useful identity 
\[
    -P_{n}'(f^{t})\cdot g^{t}=P_{n}''(f^{t})\cdot (g^{t})^{2} 
\]
It follows that

\begin{equation}
\label{eq33}
    \begin{split}
        P_{n}(f^{t})&> P_{n}(f^{t+1})-\bigg( P_{n}''(f^{t})\cdot (g^{t})^{2}+P_{n}'(f^{t})\cdot g^{t}\bigg)+\frac{1}{2}P_{n}''(\Tilde{f}^{t})\cdot (g^{t})^{2} \\
        &> P_{n}(f^{t+1})-\bigg( P_{n}''(f^{t})\cdot (g^{t})^{2}+P_{n}'(f^{t})\cdot g^{t}\bigg)+\frac{1}{2}\gamma ||g^{t}||^{2}\\
        &=P_{n}(f^{t+1})+\frac{1}{2}\gamma ||g^{t}||^{2}
    \end{split}
\end{equation}
 All other conclusions can be derived from inequality \eqref{eq33}.
\end{proof}

\section*{Appendix D. Theorem \ref{theorem-3} proof}
%\subsection*{Proof for Theorem \ref{theorem-3}:}

\begin{proof}

Note that 

\begin{equation}
    \begin{split}
P(g(f_{M,n,t}(X))\neq Y)&=P\bigg(\bigg(\frac{1}{M}\sum_{i=1}^{M}f_{Z_i,\lfloor \alpha n\rfloor,t}(X)\bigg)\cdot Y\bigg )\leq 0)\\
&=P\bigg(\sum_{i=1}^{M}T_{i} \geq \lceil \frac{M}{2} \rceil \bigg)\\
&=P\bigg( \frac{\sum_{i=1}^{M}K_{i}}{M} \geq \frac{ \lceil \frac{M}{2} \rceil -\frac{M}{2}}{M}\bigg)
    \end{split}
\end{equation}

where $T_{i}=I(f_{Z_i,\lfloor \alpha n\rfloor,t}(X)\cdot Y\leq 0), K_{i}=2T_{i}-1$ for $i=1,2,...,M$. Note that $T_{i}$ are dependent with each other by the natural of subagging. And it's easy to see that $\frac{\sum_{i=1}^{M}K_{i}}{M}$ is actually an incomplete U-statistic whose range is $\{-1,1\}$.

In addition, $P(f_{Z_i,\lfloor \alpha n\rfloor,t}(X)\cdot Y\leq 0)$ are the same for all $i=1,2,...,m$ because that  the original $n$ observations are i.i.d by assumption.  Thus we can set $P(f_{Z_i,\lfloor \alpha n\rfloor,t}(X)\cdot Y\leq 0)=p_{sub}$ for $i=1,2,...,M$. It follows that $\theta=E[\frac{\sum_{i=1}^{M}K_{i}}{M}]=E[K_{i}]=2p_{sub}-1$ and (17) can be further written as

\begin{equation}
    P(g(f_{M,n,t}(X))\neq Y)=P\bigg( \frac{\sum_{i=1}^{M}K_{i}}{M}-(2p_{sub}-1) \geq \frac{\lceil \frac{M}{2} \rceil -\frac{M}{2}}{M}-(2p_{sub}-1)\bigg)
\end{equation}

Because of the dependency among $K_{i}$, many bounding inequalities such as Hoeffding inquality can't be applied directly. Eventhough \cite{Hine_sdep} and \cite{concentration_bounds} provided a Hoeffding' inequality for the sum of dependent random variables, the conditions required by their results are hard to check in our case. Fortunately, \cite{maurer2022exponential} gave a useful concentration bound for the incomplete U-statistic with finite sample size. Before we use that,  we borrow few notations and conventions in Maurer's work.

Let $K:\mathcal{X}^{m}\to \mathbb{R}$ be a measurable, symmetric,bounded kernel where the integer $m$ 
is called degree of kernel $K$. $\theta=\theta(\mu)=E[K(X_1,...,X_{m})]$ from a sample $\mathbf{X}=(X_1,...,X_{n}) \sim \mu^{n}$, where $n$ is the number of independent
observations. $\mathcal{D}=(D_1,...,D_{M})\in \{ D\subset [n]:|D|=m\}^{M}$ is a sequence of $M$ subsets of $\{1,2,...,n\}$ with cardinality $m$ and this sequence $\mathcal{D}$ is called the design (\cite{DesignBasedIncomUstat}). Then the incomplete $U$-statistic is defined as 

\[
U_{\mathcal{D}}(\mathbf{X})=\frac{1}{M}\sum_{i=1}^{M}K(\mathbf{X}^{D_{i}})
\]
where $K(\mathbf{X}^{D_{i}})=K(X_{i_{1}},...,X_{i_{m}})$. 

Under  subagging, sequence $\mathcal{D}=(D_1,...,D_{M})$ is actually sampled with replacement from the uniform distribution on $\{ D\subset [n]:|D|=m\}$ and number $M$ is the subagging times or the number of subaged classifier we use. 

Given a design  $\mathcal{D}=(D_1,...,D_{M})\in \{ D\subset [n]:|D|=m\}^{M}$ and $k,l\in \{ 1,2,...,n\}, k\neq l$, we define

\begin{align}
\label{eq38}
\begin{split}
        R_{k}(\mathcal{D}) &= |\{i: k\in D_{i}\}|, \quad R_{kl}= |\{i: k,l \in D_{i}\}| \quad \text{where $i \in \{1,2,...,m\}$}.\\
    A(\mathcal{D})&=\sum_{k=1}^{n}\frac{R_{k}^{2}(\mathcal{D}) }{M^2}, \quad B(\mathcal{D})=\sum_{k\neq l}^{n}\frac{R_{kl}^{2}(\mathcal{D}) }{M^2}\\
    C(\mathcal{D})&=\max_{k}\frac{R_{k}}{M}
\end{split}
\end{align}

One the other hand, for a bounded kernel $K:\mathcal{X}^{m}\to [-1,1]$ and $X_1,...,X_{m},X_{1}',X_{2}'$ i.i.d with domain $\mathcal{X}$, we define

\begin{align}
\begin{split}
    \beta(K)&=E\bigg[ \bigg( V^{1}_{X_{1},X_{1}'}V^{2}_{X_{2},X_{2}'}K(X_1,,...,X_{m})^{2} \bigg) \bigg]\\
\gamma(K)&=\sup_{\mathbf{x}\in \mathcal{X}^{m},y,y'\in \mathcal{X}}E\bigg[\bigg( V^{1}_{y,y'}V^{2}_{X_{2},X_{2}'}K(x_1,,...,x_{m})^{2} \bigg) \bigg]\\
\alpha(K)&=\bigg( \sqrt{\beta(K)/2}+\sqrt{\gamma(K)}\bigg)
\end{split}
\end{align}

We will omit the dependency of design $\mathcal{D}$ and kernel $K$ when there is no ambiguity.

Then we are ready to give the concentration bound w.r.t incomplete $U$-statistics.
\newcommand{\thistheoremname}{}
\newtheorem*{genericthm*}{\thistheoremname}
\newenvironment{namedthm*}[1]
  {\renewcommand{\thistheoremname}{#1}%
   \begin{genericthm*}}
  {\end{genericthm*}}

  \begin{namedthm*}{Bernstein-type inequality for incomplete U-statistics}[\cite{maurer2022exponential}]
 For fixed kernel $K$ with values in $[-1,1]$, design $\mathcal{D}$ and $t>0$,
\begin{equation}
     P\bigg( U_{\mathcal{D}}(\mathbf{X})-\theta >t \bigg) \leq exp\bigg( \frac{-t^2}{2A\sigma_{1}^{2}+B\beta/2+(\sqrt{B\gamma}+4C/3)t } \bigg)
\end{equation}

  \end{namedthm*}

Notice that in our case, we have $U_{\mathcal{D}}(\mathbf{X})=\frac{1}{M}\sum_{i=1}^{M}K(\mathbf{X}^{D_{i}})=\frac{1}{M}\sum_{i=1}^{M}K_{i}, \theta=2p_{sub}-1$. We can assume $p_{sub}<\frac{1}{2}$ which means we require that the error probability of subbagged classifier is not higher than $\frac{1}{2}$. In generally this assumption makes sense since each of them are Adaboosted classifier which should have better performance.

Given  fixed subbagged samplings (or desgin) $\mathcal{D}$, we now have

\begin{equation}
    P(g(f_{M,n,t}(X))\neq Y |\mathcal{D}) \leq exp\bigg( \frac{-(\frac{\lceil M/2\rceil -M/2}{M}+1-2p_{sub}
    )^2}{2A\sigma_{1}^{2}+B\beta/2+(\sqrt{B\gamma}+4C/3)(\frac{\lceil M/2 \rceil -M/2}{M}+1-2p_{sub}) } \bigg)
\end{equation}

If the design becomes random, i.e $\mathcal{D}=(D_1,...,D_{M})$ is sampled with replacement from the uniform distribution of $\{ D\subset [n]:|D|=m\}$ and is independent of original dataset $\mathbf{X}$, then quantities in \eqref{eq38} are all random variables. 

Denote the events $E_{A},E_{B},E_{C}$ as following:

\begin{align}
\begin{split}
E_{A} &=\bigg \{ \mathcal{D} \in  \{ D\subset [n]:|D|=m\}^{M} : \sqrt{A(\mathcal{D})}>\sqrt{\frac{m^2}{n}}+(1+4\sqrt{ln(3/\delta)}\sqrt{\frac{m}{M}}\bigg\}\\
E_{B} &=\bigg \{ \mathcal{D} \in  \{ D\subset [n]:|D|=m\}^{M} : \sqrt{B(\mathcal{D})}>\frac{m^2}{n}+(1+4\sqrt{ln(3/\delta)}\frac{m}{\sqrt{M}} \bigg\}\\
E_{C} &=\bigg \{ \mathcal{D} \in  \{ D\subset [n]:|D|=m\}^{M} : C(\mathcal{D})>\frac{m}{n}+\frac{\sqrt{2m}+3}{\sqrt{M}}ln(3/\delta) \bigg\}\\
\end{split}
\end{align}

According to  \textbf{Lemma 4.5} and \textbf{Lemma 4.9} in (\cite{maurer2022exponential}), for $\delta>0$ and $\sqrt{M}>ln (n)$, we have

\[
P(E_{A})\leq \frac{\delta}{3},P(E_{B})\leq \frac{\delta}{3}, \quad \text{and} \quad P(E_{C})\leq \frac{\delta}{3}
\]

It follows that with probability (where randomness comes from the subagging process) at least $1-\delta$  we have:
\begin{equation}
        P_{\mathcal{P}}(g(f_{M,n,T}(\mathbf{X}))\neq Y) \leq exp\bigg( \frac{-(\frac{\lceil M/2 \rceil -M/2}{M}+1-2p_{\text{sub}}
    )^2}{2Q_{A}^{2}\sigma_{1}^{2}+Q_{B}^{2}\beta/2+(\sqrt{Q_{B}\gamma}+4Q_{C}^{2}/3)(\frac{\lceil M/2 \rceil -M/2}{M}+1-2p_{\text{sub}}) } \bigg)
\end{equation}

\end{proof}

\section*{Appendix E. Theorem \ref{theorem-6} proof}
\begin{proof}
The proof simply combines Lemma \ref{lemma-3}, Theorem \ref{theorem-4}, Theorem \ref{theorem-5} and the fact that hypothesis space consisting of PMTs has finite VC-dimension. According to Theorem \ref{theorem-4}, it suffices to show
\[
P_{S}(YF(X)\leq 0) \leq exp\bigg(-2\sum_{t=1}^{T}\gamma_{t}^{2}(B)\bigg)
\]
Suppose the given training set is $\{(\mathbf{x}_{i},y_{i})_{i=1}^{n}\}$ and the fitted classifier after $T$ rounds of AdaBoost is 
\[
F(\mathbf{x})=\sum_{t=1}^{T}\alpha_{t}f_{t}(\mathbf{x})
\]
Suppose $f_{t}$ is a PMT fitted at $t$-th step of AdaBoost, we can rewrite $f_{t}$ as defined in section 3.5:
\[
f_{t}(\mathbf{x}_{i})=\sum_{a\in A^{t}}sign[G_{a}^{t,B}(\mathbf{x}_{i})]\cdot I(\mathbf{x}_{i}\in S_{a}^{t}))
\]
where $\{S_{a}^{t}\}_{a\in A^{t}}$ is a set of partitioned feature space determined by the CART decision tree structure at $t$-th step of AdaBoost and $A^{t}$ is the corresponding index set of partitioned space.  $G_{a}^{t,B}$ is an fitted ProbitBoost model associated with the partition $S_{a}^{t}$ where $B$ represents the iteration times of ProbitBoost .

With the definition of weighted training error $\varepsilon_{t}$, we can derive the following inequality:
\newcommand{\defeq}{\overset{\mathrm{def}}{=\joinrel=}}
\begin{equation}
\begin{split}
        \varepsilon_{t}&=\sum_{i=1}^{n}w_{t}(i)\mathbf{1}\{ f_{t}(\mathbf{x}_{i})\neq y_{i} \}\\
        & = \sum_{i=1}^{n}w_{t}(i)\mathbf{1}\{ y_{i} \cdot \sum_{a\in A^{t}}sign[G_{a}^{t,B}(\mathbf{x}_{i})]\cdot \mathbf{1}\{\mathbf{x}_{i}\in S_{a}^{t}\} \leq 0\} \\
        &=  \sum_{a\in A^{t}}\sum_{i=1}^{n}w_{t}(i)\mathbf{1}\{\mathbf{x}_{i}\in S_{a}^{t}\}\mathbf{1}\{ 
          y_{i}G_{a}^{t,B}(\mathbf{x}_{i})\leq 0\}\\
        &\leq \sum_{a\in A^{t}}\sum_{i=1}^{n}w_{t}(i)\mathbf{1}\{\mathbf{x}_{i}\in S_{a}^{t}\}\frac{-ln(y_{i}G_{a}^{t,B}(\mathbf{x}_{i}))}{ln2}\\
        &= \sum_{a\in A^{t}}\sum_{\mathbf{x}_{i}\in S_{a}^{t}}w_{t}(i)\frac{-ln(y_{i}G_{a}^{t,B}(\mathbf{x}_{i}))}{ln2}\defeq \frac{\varepsilon_{t}^{P}(B)}{ln(2)}
\end{split}
\label{ineq18}
\end{equation}
By assumption, we have $\frac{\varepsilon_{t}^{P}(B)}{ln(2)}<\frac{1}{2}$ for $t=1,2,...,T$. By Theorem \ref{theorem-5}, it follows that 
\begin{equation}
    \begin{split}
        P_{S}(YF(X)\leq 0) &\leq 2^{T}\prod_{t=1}^{T}\sqrt{\varepsilon_{t}(1-\varepsilon_{t})}\leq 2^{T}\prod_{t=1}^{T}\sqrt{\frac{\varepsilon_{t}^{P}(B)}{ln(2)}\bigg(1-\frac{\varepsilon_{t}^{P}(B)}{ln(2)}\bigg)} \\
        &= \prod_{t=1}^{T}\sqrt{1-4(\gamma_{t}^{P}(B))^2}\\
        & \leq exp\bigg(-2\sum_{t=1}^{T}(\gamma_{t}^{P}(B))^2\bigg)
    \end{split}
\end{equation}
where $\gamma_{t}^{P}(B)=\frac{1}{2}-\frac{\varepsilon_{t}^{P}(B)}{ln(2)}$.
\end{proof}

\vskip 0.2in
\bibliographystyle{unsrtnat}
\bibliography{references}

\end{document}